\documentclass[twoside,11pt]{article}

\usepackage{blindtext}
\usepackage{xcolor}
\usepackage{jmlr2e}
\usepackage{cancel}

\usepackage{lastpage}
\jmlrheading{23}{2026}{1-\pageref{LastPage}}{7/26; Revised 6/26}{}{21-0000}{Jeanie Schreiber, Tyrus Berry, and Zeeshan Ahmed}

\ShortHeadings{ODIN}{Schreiber, Berry, and Ahmed}
\firstpageno{1}

\begin{document}

\title{Orthogonal Dendritic Intrinsic Networks: \\ An Architecture for Significance-Ordered, Orthogonal Latent Spaces}

\author{\name Jeanie Schreiber \email jschrei@gmu.edu \\
       \addr Department of Mathematical Sciences\\
       George Mason University\\
       Fairfax, VA 22030, USA
       \AND
       \name Tyrus Berry \email tberry@gmu.edu \\
       \addr Department of Mathematical Sciences\\
       George Mason University\\
       Fairfax, VA 22030, USA
       \AND
       \name Zeeshan Ahmed \email zeeshan.ahmed@nist.gov \\
       \addr Sensor Science Division, Physical Measurement Laboratory\\
       NIST\\
       Gaithersburg, MD 20899, USA}

\editor{}

\maketitle

\begin{abstract}

\end{abstract}
Principal Component Analysis or PCA-like properties (orthogonality, variance ranking) are seldom realized in deep autoencoder architectures. In this work, we present ODIN (Orthogonal Dendritic Intrinsic Network), a novel autoencoder architecture that recovers PCA-like latent structure in a fully non-linear regime. By incorporating a set of geometric constraints directly into the training objective, ODIN encourages latent dimensions to be mutually orthogonal and ordered by explained variance, mirroring the interpretable decomposition of PCA while retaining the expressive power of deep networks. We provide theoretical grounding for these constraints and demonstrate their compatibility with standard encoder-decoder frameworks. We also establish empirical results for both synthetic and real world datasets, establishing a principled path toward interpretable, structured feature learning and dimensionality reduction.


\section{Introduction}
Traditional autoencoders consisting of an encoder and a decoder compress data into a latent space of lower dimensionality before the decoder attempts to reconstruct the original input from the latent representation. The loss function, usually based on reconstruction error, drives the model to learn efficient latent embeddings encoding important features of the data. While effective, the standard autoencoder architecture places no structural constraints on the latent space, often utilizing the entire latent representation for inference and understanding. With no notion of ordering or separability, the resulting features are uniformly smeared across the latent space, making it challenging to interpret which latent dimensions correspond to meaningful input characteristics. 

The inherent ambiguity of latent space organization presents a common challenge in applying autoencoders for dimensionality reduction. The traditional autoencoder architecture produces latent representations with entangled, non-orthogonal directions that vary unpredictably across training runs. This instability complicates both reproducibility and interpretation.

The Orthogonal Dendritic Intrinsic Network (ODIN) architecture addresses these limitations through a key innovation inspired by biology. First, ODIN employs a dendritic structure that enforces hierarchical importance ranking of latent dimensions. By restricting decoder access to cumulative subsets of latent variables (e.g., only the first $k$ components), the network learns to prioritize directions that maximize reconstruction fidelity when progressively accrued. Additionally, strict latent space orthogonality is enforced through geometric constraints in the learning function, ensuring consistent axis alignment across training sessions. Together, these mechanisms enforce orthogonal and importance-ordered latent dimensions, mirroring PCA's eigenvalue ranking while maintaining neural networks' non-linear expressive power.

Critically, ODIN achieves this disentanglement in a completely unsupervised manner, unlike methods requiring labeled data or predefined feature hierarchies. The orthogonalized latent directions separate intrinsic attributes of the input data corresponding to statistically independent variation modes, enabling researchers to:

\begin{enumerate}
    \item Identify dominant variation modes through variance-ranked latent dimensions.
    \item Perform stable feature ablation studies by systematically removing less important axes.
    \item Achieve reproducible latent assignments across independent training runs, enabling reliable cross-experiment comparisons using consistent latent space coordinate systems.
\end{enumerate}

For AI explainability, ODIN's architecture creates an interpretable bridge between raw data and latent representations. The orthogonalized dimensions permit linear decomposition analyses typically reserved for PCA, while the non-linear encoding preserves complex relationships that linear methods might miss. This dual capability makes ODIN particularly valuable for domains requiring both model transparency and high-dimensional pattern recognition, such as biomedical signal processing or materials science characterization.

\section{Background}
Autoencoders are a  powerful tool for dimensionality reduction, offering flexibility in capturing non-linear patterns that traditional PCA cannot. A comprehensive review of autoencoder architectures and their variants by \cite{li2023comprehensive} highlights the breadth of advancements in this field. However, the review mentions a notable gap: most autoencoder designs do not attempt to replicate key features of PCA, such as orthogonality and variance-based sorting of latent space components. These properties are crucial for interpretability and systematic feature extraction, yet they remain largely unaddressed in autoencoder research.

The fundamental connection between linear autoencoders and PCA was rigorously established by \cite{plaut2018principal}, who demonstrated that single-hidden-layer linear autoencoders span principal subspaces equivalent to PCA loading vectors. However, just as PCA solutions are unique only up to orthogonal transformations, gradient descent can converge to multiple equivalent subspaces that all achieve the same minimal reconstruction error. Moreover, this equivalence breaks completely when introducing non-linear activations, leaving open the question of how to preserve PCA-like interpretability in deep architectures.

Kernel PCA may be regarded as the most natural extension of PCA to non-linear data by way of kernel transformations rather than activation functions. Kernel PCA projects data into higher-dimensional spaces where non-linear patterns become linear, enabling variance-based sorting akin to PCA while accommodating complex relationships \cite{pei2018study, majumdar2021kernelized, pei2017autoencoder}. Despite its effectiveness in handling non-linear data, Kernel PCA lacks the generative capabilities of autoencoders and requires explicit kernel selection, limiting its adaptability across diverse datasets.

Recent advances in neural implementations of PCA-like dimensionality reduction reveal persistent challenges in combining non-linear expressivity with structured latent spaces. \cite{ren2024learnable} developed an autoencoder architecture using Cayley transforms to enforce latent space orthogonality, achieving non-linear feature extraction comparable to Kernel-PCA. Parallel efforts by \cite{wang2019clustering} introduced orthogonal constraints through loss function regularization, minimizing the discrepancy between latent vector correlations and the identity matrix. While these methods successfully separate intrinsic data attributes into distinct latent components and can be effective for clustering tasks, they lack mechanisms for importance-based feature ordering.

Recent work by \cite{pham2022pca} directly addresses latent component independence and importance ordering through their PCA-Autoencoder framework. Their method introduces two innovations:
\begin{enumerate}
    \item Covariance minimization: A loss term penalizing off-diagonal elements of the latent space covariance matrix to enforce statistical independence.
    \item Sequential latent expansion: Training a progression of autoencoders where each iteration adds new latent dimensions while freezing previously learned components.
\end{enumerate}

While this approach successfully separates latent features on datasets like CelebA and synthetic ellipse images, it suffers from critical limitations. The sequential training protocol requires carefully scheduled training stages and can be sensitive to transition timing, potentially limiting global optimization. Because each stage depends on the previous one, the method is not fully end-to-end differentiable in the traditional sense and can introduce inefficiencies or overfitting if stage boundaries are not tuned precisely. More fundamentally, the implicit assumption that frozen components retain maximal variance contributions breaks down in practice, as later training stages often reallocate explanatory power to newly added dimensions.

\cite{perozo2024principal} recent introduction of POLCA Net represents a significant advancement in neural implementations of PCA-like dimensionality reduction. This architecture makes use of a multi-objective loss function that includes orthogonality loss and variance regularization terms. Similar to the work of Pham and Ladjal, the orthogonality loss term penalizes off-diagonal elements of the latent space covariance matrix to enforce independence while maintaining non-linearity. However, ordering of latent space components is achieved through statistical estimates and indirect variance maximization. In this respect, POLCA Net shares a fundamental characteristic with other variance-based regularization approaches such as the Autoencoder with Ordered Variance (AEO)~\cite{augustine2024autoencoder}, which modifies the reconstruction loss with a variance regularization term that explicitly enforces monotonically decreasing variance across latent dimensions. Through the lens of model regularization, ODIN's dendritic structure may be viewed as an implicit structural prior\footnote{See also~\cite{chen2024compressing}}, while its orthogonality penalty serves as an explicit constraint on latent geometry. We argue in the next section that this coupling yields latent representations that are both more generalizable and more interpretable in the non-linear setting than those produced by variance-based approaches.

Another approach for constraining neural network weights involves Riemannian gradient descent, which leverages Riemannian geometry and optimization techniques on matrix manifolds to generalize standard gradient descent methods~\cite{harandi2016generalized,fei2025survey}. This framework enables the intrinsic enforcement of desirable structural properties such as orthogonality, positive definiteness, or fixed-rank constraints by performing optimization directly on the constrained manifold rather than in unconstrained Euclidean space. For example,~\cite{harandi2016generalized} enforce orthogonality in neural network weight matrices by optimizing on the Stiefel manifold (the set of matrices with orthonormal columns) using retraction operators such as QR decomposition to maintain the constraint after each gradient step. While such Riemannian optimization methods can strictly enforce orthogonality during training and have demonstrated convergence guarantees in both convex and nonconvex settings, for the purposes of comparison to more standard backpropagation-based autoencoders, here we opt for simpler loss function penalties to enforce orthogonality.

An alternative school of thought for implementing PCA-like operations involves Hebbian learning, a biologically plausible network strategy that requires only pre- and post-synaptic activity correlations. The fundamental update rule follows $\Delta w\propto y\langle x_i\rangle$, where synaptic weights increase proportional to correlations between pre-synaptic $(x_i)$ and post-synaptic $(y)$ activities~\cite{olshausen1998linear}. Extensions such as Sanger's rule generalize this framework to extract multiple principal components through a hierarchical deflation procedure, sequentially learning components ordered by eigenvalue magnitude. In contrast to autoencoder approaches which impose orthogonality as explicit constraints (\emph{e.g.} cosine similarity or Stiefel layers in Riemannian Optimization), Hebbian methods learn through local correlation-based rules with exponential convergence. However, Hebbian networks are consistently outperformed by standard backpropagation-trained counterparts and remain insufficient for the demands of modern deep learning pipelines~\cite{amato2019hebbian, miconi2021hebbian, kierkegaard2023comparison}.

\subsection*{Variational Autoencoders}
Variational Autoencoders (VAEs) are widely recognized as versatile generative models and may be considered as a natural point of comparison for ODIN, primarily because of their ability to map high-dimensional data into a structured latent space. VAEs operate by transforming latent variables drawn from a simple prior distribution, typically a multivariate Gaussian, into data-like samples through a learned generative network. During training, an encoder approximates the posterior distribution over latent variables, while a decoder learns to reconstruct input data from these latent representations. These probabilistic constraints generally lead to what are called ``disentangled'' representations, in which a single latent variable is sensitive to changes in a single generative factor and is relatively invariant to changes in other factors \cite{bengio2013representation}. In particular, recent findings, such as by \cite{ichikawa2024learning}, demonstrate that VAEs begin learning entangled representations and gradually acquire disentangled latent factors during training. In addition, contemporary work by \cite{ghojogh2021factor} revealed theoretical connections between PCA, factor analysis, and variational autoencoders (VAEs). See also \cite{rolinek2019variational} who demonstrate that the diagonal posterior approximation in VAEs induces local orthogonality patterns resembling non-linear PCA.

VAEs model the data generation process as $p(x) = \int p(x|z) p(z) dz$, where the latent variable $z$ is drawn from a prior distribution $p(z)$, typically $\mathcal{N}(0, I)$. During training, exact inference of the posterior $p(z|x)$ is intractable, necessitating approximation through a variational distribution $q(z|x)$. The encoder network parameterizes this approximate posterior by outputting distribution parameters (specifically, $\mu(x)$ and $\sigma^2(x)$ for the Gaussian case) such that $q(z|x) = \mathcal{N}(z \mid \mu(x), \sigma^2(x))$. To enable gradient-based optimization while maintaining stochasticity, the reparameterization trick expresses the latent code as $z = \mu(x) + \sigma(x) \odot \epsilon$, where $\epsilon \sim \mathcal{N}(0, I)$ represents external randomness independent of the network parameters. This formulation allows differentiation through the sampling operation, as gradients can flow through $\mu(x)$ and $\sigma(x)$ while treating $\epsilon$ as constant. The training objective maximizes the Evidence Lower Bound (ELBO),
\[\text{ELBO} = \mathbb{E}_{q(z|x)}[\log p(x|z)] - \text{KL}(q(z|x)||p(z))\]
where the first term maximizes data log-likelihood and the second term (called the Kullback–Leibler (KL) divergence) penalizes deviation of the approximate posterior $q(z|x)$ from the prior $p(z)$. This stochastic framework is essential for the generative capacity of VAEs, enabling both principled uncertainty quantification and the ability to sample novel data points from the learned latent space~\cite{kingma2014semi, Liu_VAE_UQ_2025}.

In practice, VAEs are often optimized using the $\beta$-VAE objective, which introduces a tunable parameter $\beta$ to the KL divergence term in the ELBO loss to balance the two competing objectives~\cite{Higgins2016betaVAELB, higgins2017beta}. This adjustment allows for controlled trade-offs between error fidelity and latent space regularization. $\beta$-VAE's  are capable of providing significant improvements over traditional VAE's on tasks like image generation, and are today widely considered as state-of-the-art models for representation learning~\cite{ higgins2017beta, locatello2019challenging}.

However, this flexibility comes with a well-documented drawback: as $\beta$ increases, the variational posterior tends to converge toward the prior during optimization, effectively rendering the latent variables uninformative of the input data. This detrimental effect, known as ``posterior collapse'', undermines the expressive power of the latent space and limits the effectiveness of VAEs in representation learning tasks~\cite{He2019LaggingIN}. VAEs struggle particularly with capturing low-variance directions, as these dimensions are most susceptible to collapse. Identifying an appropriate $\beta$ value remains a nontrivial challenge, as this hyperparameter must be carefully tuned to achieve an effective compromise between two competing objectives~\cite{fil2021beta}.

In addition, while implicit probabilistic constraints enable partial disentanglement, they do not enforce any explicit ordering of latent variables. The learned latent space exhibits permutation symmetry, meaning dimensions can be arbitrarily reordered without affecting reconstruction quality. Empirically, this manifests as inconsistent feature assignments across training runs. For example, dimension 3 in one instantiation might capture a high-importance feature, while in another run this same information migrates to dimension 5. This inconsistency severely limits utility for comparing results across experiments or performing systematic ablation studies.

Despite their prominence in generative modeling and representation learning, VAEs are not natively designed for supervised downstream tasks such as regression, and their application in these settings remains comparatively under-explored~\cite{zhao2019variational}. A key contributing factor is that the latent representations learned by VAEs are not guaranteed to be discriminative or predictively informative; uninformative latent features are a well-documented failure mode that arises when the balance between reconstruction and regularization is not carefully managed~\cite{zhao2017towards}. Benchmarking studies have also demonstrated that the large number of hyperparameters requiring tuning in VAE-based models (including $\beta$) directly impacts downstream predictive performance, making consistent and reliable results difficult to achieve across varied datasets and tasks~\cite{eltager2023benchmarking}.

ODIN, by contrast, avoids excessive hyperparameter tuning and achieves latent space disentanglement through architectural constraints rather than probabilistic regularization. Moreover, the dendritic decoder structure provides explicit variance-based hierarchical ranking of latent features, eliminating permutation symmetry and allowing for consistent, reproducible latent mode assignments across independent training runs.  

\section{ODIN}

The architecture introduced in ODIN modifies the conventional autoencoder setup by incorporating multiple `dendritic' decodings alongside the standard full decoding. Each dendritic decoding is restricted to a cumulative subset of the full latent representation. The first decoding reconstructs the data using only the first latent component, the second decoding uses the first two, and so on until all components are included ( Figure \ref{fig:Dend Net}). Specifically, given a dataset $X\in \mathbb{R}^{N\times n}$ consider an encoding into a $k$-dimensional latent space, yielding latent representation $Z = [z_1, z_2, \dots z_k]\in \mathbb{R}^{N\times k}$. For each $j \in \{1,\dots,k\}$, let $Z_{[j]} = Z[:,1:j] = [z_1, \dots, z_j]$ denote the first $j$ columns of $Z$. If $Dec$ denotes a decoder mapping from the latent space to the data space, then the $j^{th}$ reconstruction $\hat{X}_j$ produced by the decoder is defined as:
\begin{equation}
    \hat{X}_j = Dec(Z_{[j]}).
\end{equation}
In other words, at decoding step $j$, the decoder is presented with only the first $j$ learned features so that $\hat{X}_j$ approximates $X$ using only the first $j$ latent components.

\begin{figure}[h!]
    \centering
    \includegraphics[width=\linewidth]{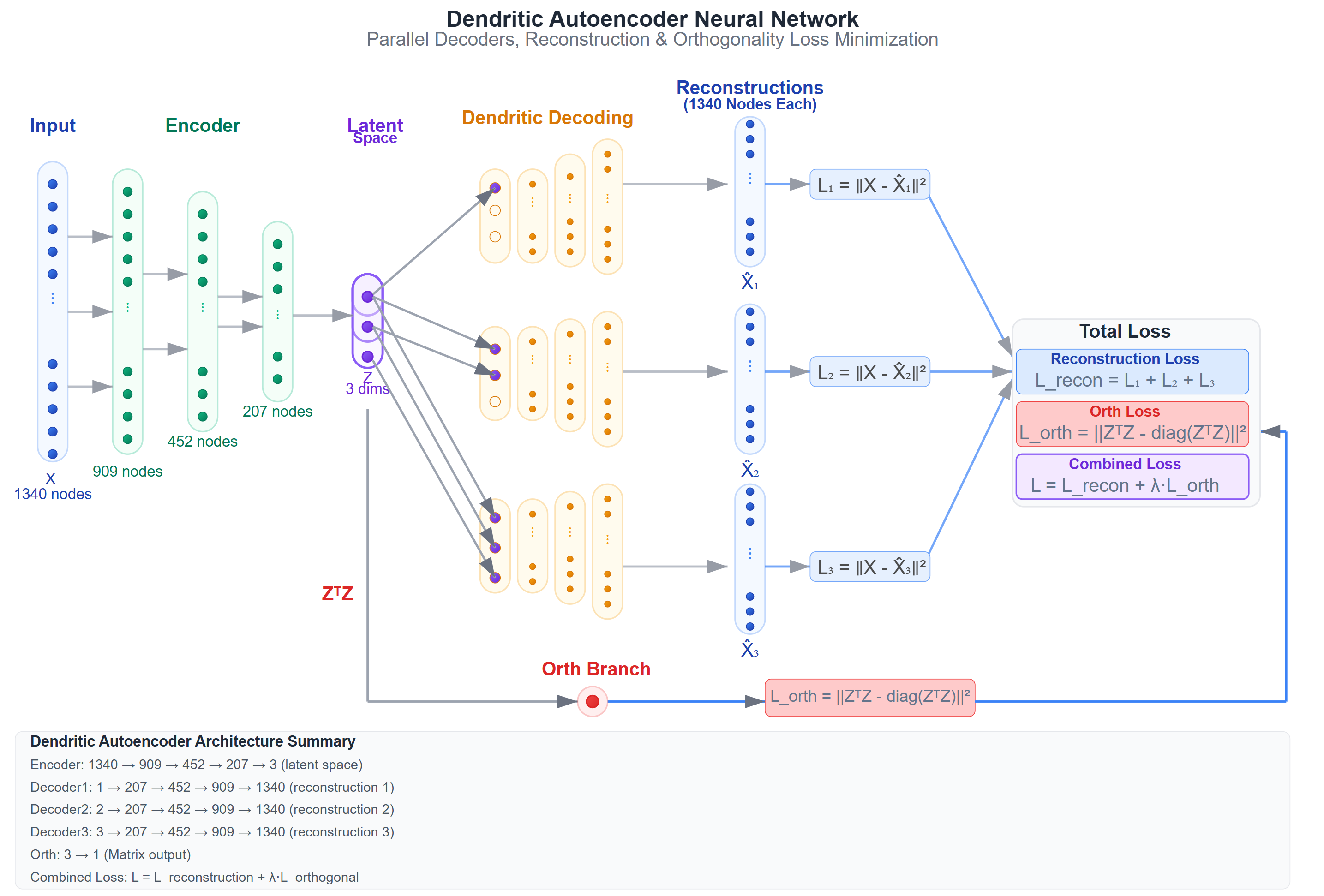}
    \caption{Architecture of ODIN. The encoder maps input data to a latent representation $Z = [z_1, \dots, z_k]$, from which each dendritic branch produces a reconstruction using cumulatively larger number of latent variables. The total loss combines each reconstruction term with an orthogonality penalty that encourages de-correlation among latent dimensions.}
    \label{fig:Dend Net}
\end{figure}

Each dendritic decoding in the ODIN network produces output reconstructions of the same size as the original input data, regardless of how many latent dimensions it accesses. During training, the objective error (according to the loss function) is calculated for each dendritic decoding individually, and these errors are then aggregated into a combined loss. Specifically, given reconstructions $\hat{X}_j$ from subsets $Z_{[j]}$ of the latent space, the mean squared error aggregate is
\begin{equation}
    \mathcal{L}_{\text{Dend}} = \sum_{j=1}^k L_j =  \sum_{j=1}^{k} \mathcal{L}_{\text{MSE}}(X, \hat{X}_j)
\end{equation}
This total loss guides the backpropagation process, forcing each decoding to maximize reconstruction quality given its available features. Through this hierarchical setup, the network learns to allocate representational capacity such that earlier decodings capture the most meaningful variation, enforcing an implicit ordering of latent importance such that the earliest are most critical for faithful reconstruction. In effect, this approach mimics the PCA variance sorting in which the principal components give a sorted set of directions with maximal variance.

Although the total $\mathcal{L}_{\mathrm{Dend}}$ error used during backpropagation increases with the introduction of dendritic loss terms, bounding it away from zero (see appendix~\ref{app:ODIN}), faithful reconstructions remain attainable by relying on the output of the final (full) decoding. When given access to all latent variables (\emph{i.e.}, all degrees of variation) the full decoding can, in principle, reconstruct the entire data distribution with a theoretical lower bound of zero reconstruction error. This ensures that, despite the added regularization, the ODIN model continues to produce reconstructions that are both highly accurate and visually consistent with the original data.

To implement dendritic decodings in practice, there are generally two approaches. The first uses an ensemble of $k$ independent decoders ${Dec_1, Dec_2, \dots, Dec_k}$, where decoder $Dec_j$ receives a fixed $j$-dimensional input and produces reconstruction $\hat X_j = Dec_j(Z_{[j]})$. The second approach employs a single shared decoder $Dec$ called $k$ times on cumulatively zero-masked versions of the full latent representation, \emph{i.e.} $\hat X_j = Dec(Z_{[j]}\oplus \mathbf{0}_{j+1:k})$, where $\oplus$ denotes concatenation with the zero-padding matrix $\mathbf{0}_{j+1:k}\in \mathbb{R}^{N\times (k-j)}$.

In informal comparisons across both linear and non-linear experiments, we observe empirically that the single-decoder formulation achieves similar reconstruction performance to the independent decoder ensemble while reducing the total number of trainable parameters by a factor of $k$. This parameter efficiency improves training stability, accelerates convergence, and facilitates scaling to higher latent dimensionalities without excessive memory overhead.

To further emulate PCA principle components, the cumulative structure of dendritic decoding works in concert with orthogonality constraints to enforce non-overlapping features and further stabilize axis alignment. Following \cite{perozo2024principal}, ODIN introduces an explicit orthogonality loss term:
\begin{equation}
    \mathcal{L}_{\text{orth}} = \sum_{i<j}S_{ij}^2
\end{equation}
Where $S_{ij}$ is the $(i,j)^{th}$ component of $S = Z^TZ$. This term penalizes off-diagonal entries in the latent similarity matrix, thereby enforcing feature disentanglement. In addition, we incorporate a moving average into this process to account for batching of large datasets. This geometric approach ensures that latent dimensions remain statistically independent without relying on probabilistic approximations. 

The total training objective is thus defined as a multi-objective loss,
\begin{equation}
\mathcal{L}(W,X) = \mathcal{L}_{\text{Dend}}+ \lambda_{\text{orth}} \cdot \mathcal{L}_{\text{orth}}
    \label{eq:total_loss}
\end{equation}
where $\lambda_{\text{orth}}$ weights the orthogonality constraint. The result is a latent space whose dimensions are both disentangled and importance-ranked, facilitating downstream tasks such as ablation studies and robust explainability analyses.

In terms of latent space ordering, ODIN adopts architectural restrictions similar to~\cite{pham2022pca} sequential training approach but eliminates the need for training multiple networks or freezing previously learned components. ODIN achieves significance-based ordering within a single, jointly trained network where structural restrictions on decoder access encourage the model to prioritize the most informative latent dimensions.

Compared to the $\beta$-VAE’s method of enforcing independence probabilistically, ODIN's approach with $\lambda_{\text{orth}}$ directly penalizes correlations in the latent space via measurable covariance properties. This makes it less prone to failures like posterior collapse and more robust to changes in the parameter. In practice, we found tuning the parameter $\lambda_{\text{orth}}$ to be less sensitive compared to $\beta$ in $\beta$-VAE's, making it easier to adjust. Since $\lambda_{\text{orth}}$ directly scales the orthogonality loss, we often start with a moderate value $\approx 1$ and then adjust it based on the relative importance of orthogonality versus reconstruction accuracy. Larger batch sizes and datasets with higher intrinsic correlations usually require increasing $\lambda_{\text{orth}}$ to maintain effective decorrelation in the latent space.

We also remark that it is possible to augment a VAE with dendritic-style decoding to induce an importance ordering over latent factors. However, we found this approach to be less robust in practice than ODIN. Moreover, we were motivated to find a method that provably generalized PCA in the sense that it recovered PCA when applied to an affine network.  We were able to obtain these theoretical results for ODIN as shown below.

Exploring a dendritic approach to VAEs remains an interesting direction for further research.  On one hand, the VAE framework offers well-established theoretical grounding for latent space regularization and has been shown to be effective at disentangling latent dimensions in its own right~\cite{higgins2017beta, rolinek2019variational}. On the other hand, a direct cosine-similarity regularizer offers a transparent mechanism for enforcing orthogonality among latent directions, without requiring careful selection of the $\beta$ parameter across datasets and regimes. In practice, we found that combining VAE-style regularization with dendritic decoding leads to a significant degradation in the orthogonality otherwise present in either VAE or ODIN separately. As such, for practitioners who simultaneously desire both importance ordering and latent disentanglement, ODIN represents a more natural and robust choice under the current framework.

\subsection{The True Loss Function for PCA}

The standard linear algebra view of PCA takes the form: \begin{equation}
    \max_W  \text{Tr}(W^\top X^\top X W), \quad \text{subject to  } W^\top W = I_{k\times k}.
    \label{eq:PCA_objective}
\end{equation}
Note that $\text{Tr}(W^T X^T X W) = ||XW||_F^2$ is the variance of the projection into the latent space. For data $X\in \mathbb{R}^{N\times n}$, PCA finds an orthogonal matrix $W = [w_1, w_2, \dots, w_k]\in \mathbb{R}^{n\times k}$ whose columns span the directions of maximal projected variance. The following standard result characterizes the solution set of (\ref{eq:PCA_objective}) being any projection onto the span of the first $k$ right singular vectors of $X$ (for completeness the proof is included in Appendix~\ref{app:PCA}).

\begin{theorem}\label{PCA1}
    Let $X \in \mathbb{R}^{N \times n}$, with singular value decomposition $X=USV^\top$, then the solution to equation~\ref{eq:PCA_objective} is $W=V_{[k]}Q=VI_{n\times k}Q$ where $Q \in \mathbb{R}^{k\times k}$ is any orthogonal matrix.
\end{theorem}
While Theorem \ref{PCA1} is typically taken as the theoretical backbone of PCA, it leaves an important degree of freedom unresolved: the solution $W$ is only unique up to an arbitrary orthogonal rotation $Q$. In this framework, the individual columns of $W$ need not align with the principal components $v_1, \dots, v_k$, nor are they ordered by importance. Here we address both issues by shifting to a reconstruction-based objective rather than variance in the projected latent space. We show that minimizing this new loss not only recovers the PCA subspace but also guarantees ordering within the projection by importance.  While achieving ordering in the PCA algorithm is trivial, it is critical that the loss function capture this ordering so that we can use the loss function to generalize PCA to the nonlinear setting.

As a starting point, we recall the well-known equivalence between the maximum-variance and minimum-reconstruction formulations of PCA (the proof is in appendix~\ref{app:PCA}):
\begin{proposition}
    For any $X\in \mathbb{R}^{N\times n}, W\in \mathbb{R}^{n\times k}$ we have the following equivalence:\[\max_{W^\top W = I}  \text{Tr}(W^\top X^\top X W) = \min_W \| X^\top - W W^\top X^\top \|_F^2\]
\end{proposition}

In other words, the maximum trace PCA objective is equivalent to minimizing the squared residuals of the orthogonal projections of the data $\sum_{i=1}^N \|x_i-p_i\|^2$,  where $p_i = WW^\top x_i$ is the projection of $i$'th data point $x_i$ onto the $k$-dimensional subspace spanned by the columns of $W$. Both problems share the same solution set, recovering the subspace spanned by the top $k$ principal components but, as noted above, without pinning down the orientation within that subspace.


To resolve this ambiguity and recover the principal components in their natural order, we break the rotational symmetry of the solution space by introducing a sequence of nested reconstruction terms. Specifically, we augment the objective with terms involving $XW_{[i]}$ for $i=1, \dots, k$, where $W_{[i]}$ denotes the submatrix formed by the first $i$ columns of $W$. Each such term measures how well the first $i$ directions reconstruct $X$, penalizing solutions that fail to prioritize higher-variance directions first. This is exactly the dendritic idea in the linear setting; since each $XW_{[i]}$ represents the signal available to the $i$-th dendrite for reconstructing $X$. The resulting objective becomes,
\[\max_{W^\top W = I_{k\times k}} \sum_{i=1}^k ||X W_{[i]}||_F^2 =  \min_W \sum_{i=1}^k \|X^\top-W_{[i]}W_{[i]}^\top X^\top\|_F^2. \]

Under the assumption that the nonzero singular values of $X$ are distinct, the difference between optimizing over the full $k$-dimensional subspace compared to optimizing over each individual subspace simultaneously becomes explicit. When $s_1 > s_2 > \cdots > s_k > 0$,
each optimal $i$-dimensional principal subspace is uniquely determined as
$\mathrm{span}(V_{[i]})$. Since there is no rotational freedom within any degenerate
right singular space, the only way to simultaneously satisfy
\begin{equation}
    \mathrm{span}(W_{[i]}) = \mathrm{span}(V_{[i]}) \quad \text{for all } i = 1,
    \ldots, k,
\end{equation}
is for the columns of $W$ to align exactly with $v_1, \ldots, v_k$ up to individual
sign changes. That is, the distinct spectrum forces the rotational degrees of freedom
to collapse entirely, leaving sign flips as the only residual freedom in the joint
minimizer. 

The following theorem characterizes the minimizer as the correctly ordered PCA variables, with proof provided in appendix~\ref{app:PCA}:
\begin{theorem}
    Let $X \in \mathbb{R}^{N \times n}$, with singular value decomposition $X=USV^\top$ and distinct singular values $s_1>s_2>\cdots>s_k>0$, then 
\[ k s_1^2 + (k-1)s_2^2 + \cdots + 2 s_{k-1}^2 + s_k^2 = \min_{W \in \mathbb{R}^{n \times k}} \sum_{i=1}^k \|X^\top-W_{[i]}W_{[i]}^\top X^\top\|_F^2. \]
and the minimum is achieved when $W=V_{[k]}\Sigma=VI_{n\times k}\Sigma$, where $\Sigma = \mathrm{diag}(\varepsilon_1, \dots, \varepsilon_k)$, with $\varepsilon_j = \pm 1$.
\label{thm:dend_loss}
\end{theorem}

\subsection{Equivalence to PCA}
Having established the reconstruction-based formulation of PCA and its ordered solution, we now turn to the central theoretical result of this work. That is, in the linear regime, ODIN is provably equivalent to PCA. Specifically, we show that when the encoder and decoder are constrained to linear maps, the ODIN training objective reduces exactly to the nested reconstruction loss derived above, and the resulting weight matrices converge to the principal components of the data in their natural order.

Consider that, for a purely linear autoencoder, the \emph{encoder} of the network can be represented as \( Z = X W + b^{\text{enc}} \), where \( W = [w_1, w_2,\dots, w_k] \in \mathbb{R}^{n \times k} \) is the weight matrix of the encoder, with bias term \( b^{\text{enc}} \in \mathbb{R}^k \). Likewise the \emph{decoder} reconstructs the input as \( \hat{X} = Dec(Z)= Z W^\top + b^{\text{dec}} \), where \( \hat{X} \in \mathbb{R}^{N \times n} \) is the reconstructed data and \( W^\top \in \mathbb{R}^{k \times n} \) is the weight matrix of the decoder.

To speed up convergence we mean-center our data prior to encoding, rendering the bias terms identically zero so that we may omit them without loss of generality~\cite{plaut2018principal}. The mappings simplify to: \[ Z = X W, \quad \hat{X} = ZW^\top = XWW^\top. \]
Using mean squared error objective, the loss function then becomes:
\begin{align*}
    \mathcal{L}_{\text{MSE}}(X,\hat{X}) &= \min_W \|X-XWW^\top\|^2\\
    &= \min_W \|X^\top-WW^\top X^\top\|^2.
\end{align*}

As outlined above, ODIN augments the standard autoencoder loss function with dendritic decodings $\hat{X}_i = Dec(Z_{[i]})$ produced by zeroing out columns of $Z$. That is, we have $Z_{[i]} = [XW]_{[i]} = X{[W]_{[i]}}$, and thus:
\[\hat{X}_i= XW_{[i]}W^\top = X\sum_{s=1}^iw_sw_s^\top = XW_{[i]}W_{[i]}^\top.\] The total reconstruction loss is then:
\begin{align*}
    \mathcal{L}_{\text{Dend}}= \sum_{i=1}^k\mathcal{L}_{\text{MSE}}(X,\hat{X}_i) &= \min_W \sum_{i=1}^k\|X-XW_{[i]}W_{[i]}^\top\|^2 \\&= \min_W \sum_{i=1}^k \|X^\top-W_{[i]}W_{[i]}^\top X^\top\|_F^2
\end{align*}

Thus we have shown that the loss objective of a dendritic autoencoder (with no orthogonality constraints) is equivalent to ordered PCA and recovers the true PCA solution in the purely linear case. 

It is interesting to note that, in place of dendritic decoding, an alternative way to recover principal components in the linear case is to augment standard reconstruction minimization with an orthogonality penalty, \emph{i.e.}
\[\mathcal{L} = \mathcal{L}_{\text{MSE}}(X,\hat{X})+\mathcal{L}_{\text{orth}}(Z).\] This approach also recovers the columns of $V$, but without any particular ordering. See appendix~\ref{app:ODIN} for a proof. 

The linear analysis reveals that the goals of orthogonality regularization and dendritic reconstruction are closely related in the linear regime. Both recover the principal components, but only the dendritic objective guarantees that components are ordered by importance. In fact, the combined objective incorporating both dendritic reconstruction and orthogonality regularization, \emph{i.e.} 
\begin{align*}
    \mathcal{L}_{\text{ODIN}} \equiv \mathcal{L}_{\text{Dend}}+ \mathcal{L}_{\text{orth}}&=  \min_W \left[ \sum_{i=1}^k \|X^T-W_{[i]}W_{[i]}^TX^T\|_F^2+ \sum_{i<j}(W^TX^TXW)_{ij}^2\right]\\&=\min_W \left[ \sum_{i=1}^k \|X^T-W_{[i]}W_{[i]}^TX^T\|_F^2+ \frac{1}{2}\|W^TX^TXW-\text{diag}(W^TX^TXW)\|_F^2\right],
\end{align*}
shares the same minimizer as dendritic reconstruction alone in the linear regime (see the proof in appendix~\ref{app:ODIN}):
\begin{theorem}
    Let $X \in \mathbb{R}^{N \times n}$, with singular value decomposition $X=USV^\top$ and distinct singular values $s_1>s_2>\cdots>s_k>0$, then 
    \[ k s_1^2 + (k-1)s_2^2 + \cdots + 2 s_{k-1}^2 + s_k^2 = \min_{W \in \mathbb{R}^{n \times k}}  \mathcal{L}_{\text{Dend}}+ \mathcal{L}_{\text{orth}}, \]
    and the minimum is achieved when $W=V_{[k]}\Sigma=VI_{n\times k}\Sigma$, where $\Sigma = \mathrm{diag}(\varepsilon_1, \dots, \varepsilon_k)$, with $\varepsilon_j = \pm 1$.
\label{thm:odin_loss}

\end{theorem}

In the non-linear regime, however, this similarity is no longer guaranteed, and the two objectives could potentially be considered as independent design choices serving distinct purposes. The dendritic structure enforces concentration of information across the latent dimensions, while the orthogonality penalty promotes decorrelated representations. By default, we adopt the joint loss function as a principled extension of the linear theory in all non-linear variants of ODIN, but the two can in principle be decoupled. Dendrites alone suffice when ordered concentration is the primary goal, and the orthogonality term should be added only when decorrelated latent variables are explicitly desired.

\subsection{Relationship between ODIN and POLCA/AEO}

Our analysis of ODIN is grounded in the equivalence between PCA's maximum variance criterion and minimum reconstruction loss for linear networks, a result we exploit to derive several properties of ODIN's loss minimizer. These results establish that ODIN's linear special case is not merely PCA-like but is provably optimal in the same sense that PCA is optimal.

The critical question then is how the approach generalizes when linearity is relaxed. Variance-based methods such as AEO~\cite{augustine2024autoencoder} and POLCA-Net~\cite{perozo2024principal} carry their ordering principle explicitly into the non-linear setting by retaining direct penalties on activation variance or covariance statistics in the latent space. This is a natural extension in form, but not in function: Variance, as a second-order statistic defined in the ambient Euclidean space, does not in general correspond to intrinsic geometric importance along directions of a non-linear manifold~\cite{abuqrais2026riemannian}. As the degree of non-linearity increases, the variance of latent activations becomes an increasingly indirect proxy for the information content of each latent dimension, and the ordering it induces loses its connection to any well-defined optimality criterion analogous to PCA. This problem is further compounded by the fact that a sufficiently expressive decoder can rescale latent dimensions arbitrarily (a dimension assigned low variance by the encoder may carry high reconstructive weight, and vice versa), rendering the measured variance in the latent space a gauge-dependent quantity with no intrinsic meaning. The regularization term remains well-defined as a scalar penalty, but its geometric interpretation, and therefore the interpretability of the resulting latent space, degrades.

ODIN takes a different route. Rather than extending the variance criterion directly, it retains the \emph{reconstruction} route to importance ordering and instead relies on the dendritic architecture to break the symmetry of the solution space in a manner that recovers ordering as a structural consequence of the model class. This architectural prior imposes an ordering that is intrinsic to the model rather than appended to it through a statistical penalty, and it does so in a way that remains geometrically meaningful regardless of the degree of non-linearity. The orthogonality penalty then operates on a solution space that is already structured, enforcing independence between the components that the dendritic architecture has already separated. We argue that this division of labor (architecture for ordering, loss terms for independence) is both more principled and more interpretable than asking potentially numerous scalar loss penalties to accomplish similar objectives in a regime where geometric meaning is no longer guaranteed.

\section{Experiments}
The goal of this section is to systematically evaluate the ODIN architecture in comparison to both standard and variational autoencoder models, with direct reference to Principal Component Analysis (PCA) where appropriate. To establish a performance baseline, in section 4.1 we begin with experiments on synthetic 3D Gaussian point cloud data, a setting in which PCA is known to provide the optimal solution for linear dimensionality reduction. We verify that the ODIN architecture is capable of recovering the principal components given by PCA, confirming its correct behavior in simple, well-understood scenarios.

Having validated ODIN's ability to recover principal components in the controlled setting of linear transformations, we then examine its performance on the more complex dataset of MNIST handwritten digits. The objective in this experimental setting is not to exactly replicate PCA's solution, as doing so would constrain the model to linear relationships and forfeit the expressive advantages of deep architectures. Instead, we seek to demonstrate that ODIN achieves PCA-like properties—specifically, orthogonal latent dimensions with explicit importance ordering—while simultaneously leveraging non-linear expressivity to capture complex features inherent in human handwriting. This benchmark dataset provides an ideal testbed for evaluating ODIN's capacity to capture non-linear structure while maintaining interpretability and consistency. 

Finally, in section 4.3 ODIN is applied to a real-world scientific task of analyzing NV diamond's photoluminescence (PL) spectroscopy data. In this domain, interpretability and explainability of latent representations are especially crucial, as the ability to resolve distinct latent modes has direct implications for ability to develop data-driven strategies for material characterization and development of novel sensors. Comparisons in this section emphasize ODIN’s unique strengths in exposing latent dimensions that correspond to meaningful physical changes e.g. changes in lattice spacing, outperforming traditional networks in terms of scientific insight and practical utility.

\subsection*{Latent Modes} 
Throughout this work, we refer to latent modes as the characteristic directions of variation captured by individual latent dimensions. Mathematically, these modes can be extracted as 
\begin{equation}
    R = Z^TX
\end{equation}
where $Z\in \mathbb{R}^{N\times k}$ is the $k$-dimensional latent representation for $N$ data points, and $X\in \mathbb{R}^{N\times n}$ is the matrix of original input data with $n$ features for each sample. Each row of $R$ represents the projection pattern that a specific latent dimension encodes. In ODIN, these modes are both independent and importance-ranked, enabling researchers to identify which patterns account for the most significant data variations.

\subsection{Point Cloud Data}

In the first set of experiments, we investigate ODIN’s ability to recover principal components using a three-dimensional Gaussian point cloud as a controlled, synthetic dataset. The data for this experiment are generated by sampling $10,000$ points from a standard three-dimensional normal distribution, which is then linearly transformed by applying rotation and anisotropic scaling. This procedure produces an ellipsoidal point cloud in three-dimensional space, where each principal axis corresponds to a distinct variance direction dictated by the stretching operation.

PCA is provably optimal for such datasets because it identifies orthogonal directions in the data that maximize variance, corresponding exactly to the axes of the ellipsoid after transformation. For any dataset produced by a linear transformation of Gaussian noise, PCA provides the best linear dimensionality reduction and recovers the original transformation up to orthogonal rotation. As such, this synthetic scenario allows for direct and unambiguous benchmarking by comparing against the ground truth solution given by the principal axes computed by PCA.

For this experiment, the base `standard' autoencoder architecture is constructed as a purely linear network, consisting of a linear encoder followed by a linear decoder, and explicitly omitting any non-linear activation functions. In this setting, the latent variables and output reconstructions are just linear functions of the input, which ensures that the network operates within the space of linear transformations. This is desirable as the underlying structure of the ellipsoidal point cloud is itself defined completely by linear transformations.

To implement ODIN, the standard network is modified by introducing dendritic decodings, each producing reconstructions using different cumulative subsets of the latent variables. We also introduce the orthogonality loss term in equation~\ref{eq:total_loss} to enforce independence of latent dimensions, although it is not necessary in the linear setting. Indeed we observe experimentally that dendritic decoding alone was sufficient to achieve feature disentanglement and hierarchical importance ordering in the linear regime, though often taking longer to converge to the global minimum when including the orthogonality loss term.

The evaluation of each model is made possible by comparing their extracted latent modes to the principal directions obtained from PCA. To assess alignment with the principal component axes, we utilize a three-dimensional singular value decomposition (SVD) representation:
\begin{equation}
    X = US V^T
\end{equation}
where $U\in \mathbb{R}^{3\times 3}$ contains the principal components (directions of maximal variance), facilitating a direct comparison with the extracted latent modes. The cross-correlation matrix $M$ between the computed latent modes $R$ and the PCA principal component matrix $U$ is given by:
\begin{equation}
    M = RU
\end{equation}
We measure how far this cross-correlation matrix deviates from the identity, quantifying the error as:
\begin{equation}
    e = ||M-I||_F
    \label{eq:correlation_error}
\end{equation}
where $I$ is the $3\times 3$ identity matrix and $||\cdot ||_F$ denotes the Frobenius norm.

\subsubsection*{Results}

For the standard linear autoencoder, results reveal a significant lack of consistency in the orientation of the learned latent modes across different training runs. Empirically, the first latent mode aligns reasonably well with the first principal direction given by PCA in two out of seven runs; however, the second and third latent modes are frequently miscalculated, as lower-variance directions are often subsumed within higher-variance axes. This can be attributed to the comparatively small variance in the third principal direction and the unordered, entangled nature of the learned latent modes. This instability is evident when examining the cross-correlation of the learned latent directions with the PCA basis: the resulting alignment varies markedly between model initializations. See Figure~\ref{fig:PC_latent_modes}.

In contrast, the ODIN model demonstrates remarkable stability and reproducibility in the orientation of its latent modes across all training runs. For every model initialization and training instance, the extracted latent directions are nearly identical and exhibit precise alignment with the principal component axes prescribed by PCA. The cross-correlation error between ODIN's latent modes and the PCA basis is consistently close to the identity matrix, yielding minimal Frobenius norm error. See Figure~\ref{fig:correlation_error}. Thus we can say with confidence that ODIN matches PCA’s optimal solution in this linear regime.

In parallel experiments, variational autoencoders (VAEs) were evaluated as a leading competitor to ODIN in terms of latent space structure. In our simple synthetic 3D Gaussian experiments, we found that VAE's consistently struggled to capture directions of low variance due to posterior collapse. This limitation manifests as a pronounced flattening of dimensions with subtle variation, effectively disregarding them. Lowering the $\beta$ parameter in the ELBO loss can partially mitigate this issue, allowing some recovery of low variance directions at the cost of reduced overall disentanglement. In our experiments, we found that after tuning the $\beta$ parameter it was possible to get reconstructions on par with those from ODIN; however, the latent representations still suffered from variability across training runs. While the probabilistic constraints enforcing disentanglement led to the emergence of three distinct latent modes that qualitatively corresponded to the principal component directions, these modes lacked any inherent ordering or hierarchy. Each latent dimension contributed equally to the representation, resulting in frequent permutation of modes between training instances. See figure~\ref{fig:PC_latent_modes}. This frequent interchange of latent dimensions is especially problematic when recalibrating models performing downstream tasks using only a subset of the latent space.

Another interesting experiment was conducted using non-linear point cloud data, presenting a more challenging scenario for latent space analysis. In this setting, both ODIN and VAE architectures were augmented with activation layers\footnote{We opted to use leaky ReLU activation functions in place of standard ReLU for the VAE network to address the issue of suppressed gradients in low-variance dimensions, which can lead to the complete zeroing out of subtle latent features.} to enhance their non-linear expressivity and better accommodate the data’s underlying structure. Despite these modifications, VAEs consistently failed to fully disentangle latent factors, regardless of the degree of variation in the data or the level of randomness introduced in the model. The latent variables in VAE's remained entangled, and reconstructions often appeared distorted or failed to reflect the true manifold of the original data. In contrast, ODIN demonstrated robust disentanglement and successfully sorted the latent variables according to their importance, all while maintaining faithful and interpretable reconstructions. These results highlight ODIN’s unique ability to manage both non-linearity and interpretability in its latent representations, outperforming VAEs in this more complex regime.

\begin{figure}[h!]
    \centering
    \includegraphics[width=\linewidth]{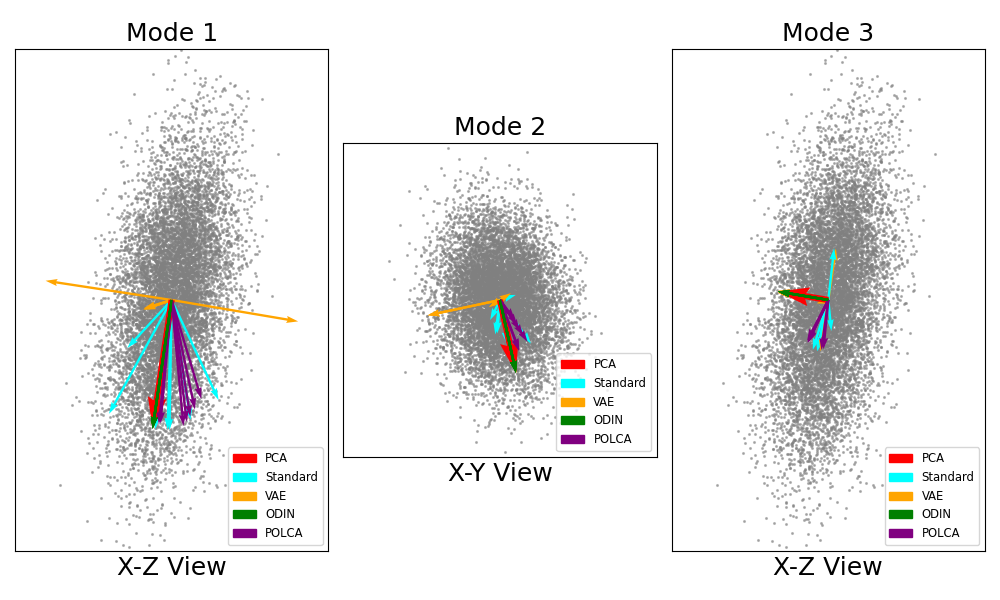}\\
    \caption{Four different autoencoder architectures are trained and evaluated seven times on a three-dimensional point cloud dataset. Cross-sectional views of raw point-cloud data are shown alongside PCA principal component directions in red. The corresponding latent modes as resolved by each network are shown alongside the PCA vectors, demonstrating high variability across training runs for both standard AE and VAE approaches. POLCA comes much closer to recovering the first two latent modes but has some variability and misses the third, lowest variance, latent mode.  In contrast, the latent directions learned by ODIN consistently align with the ground truth directions given by PCA.}
    \label{fig:PC_latent_modes}
\end{figure}
\begin{figure}
    \centering
    \includegraphics[width=\linewidth]{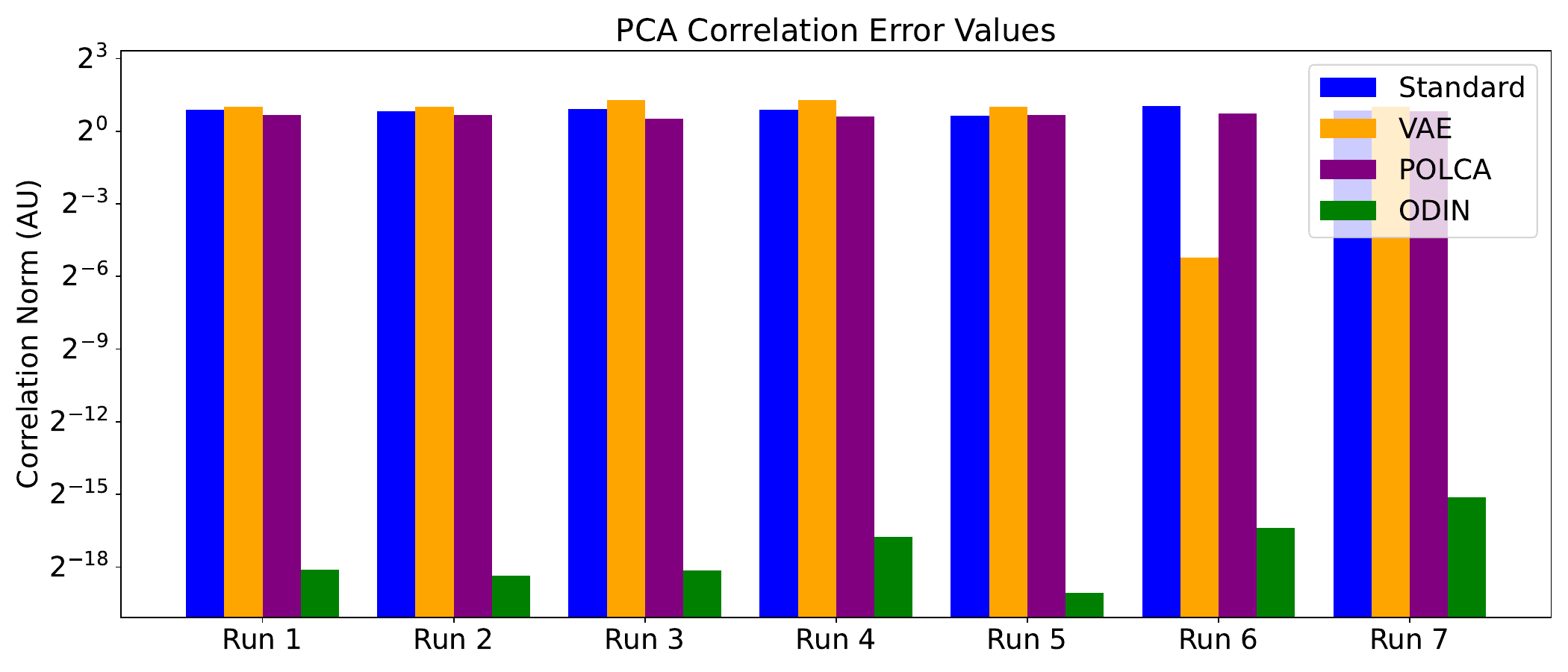}
    \caption{The error in the cross-correlation matrix as defined in equation~\ref{eq:correlation_error} for four different models across seven independent training runs. In each instance, the ODIN cross-correlation matrix is near-identity resulting in Frobenius norm error close to zero. This confirms that ODIN recovers the principal components in their correct importance order regardless of random initialization.}
    \label{fig:correlation_error}
\end{figure}

\subsection{MNIST Data}
The next dataset we evaluate is the MNIST set of handwritten digits, restricted to digits $1$ and $2$ for our primary experiments. In this setting, all three architectures under comparison—the standard autoencoder, the $\beta$-VAE, and the ODIN architecture—employ convolutional neural network structures, which represent the standard approach for processing image data. The encoder consists of two convolutional layers with $32$ and $64$ filters respectively, each followed by ReLU activations. The resulting feature maps are flattened and passed through a fully connected layer that projects to a $k=7$-dimensional latent space.

The decoder architectures mirror this structure in reverse. For the standard conv-AE and conv-$\beta$-VAE models, a single decoder maps from the full $7$-dimensional latent space through a fully connected layer to reshape into appropriate dimensions for transposed convolution (deconvolution) layers. These layers progressively upsample and reduce channel depth to reconstruct the original $28\times28$ pixel images. The ODIN architecture modifies this structure by incorporating $7$ dendritic decodings, where each decoding $\hat X_j$ receives only the first $j$ latent dimensions $Z_{[j]}$ and produces a full $28\times28$ reconstruction.

Training was conducted using the Adam optimizer with stepped learning rate for 2500 epochs on batches of 480 samples. For the $\beta$-VAE, we experimented with $\beta$ values ranging from $0$ to $1$, ultimately selecting $\beta=1$ as providing the best balance between reconstruction quality and disentanglement. The ODIN orthogonality weight was set to $\lambda=1$ after preliminary experiments showed this value effectively enforced latent orthogonality without overwhelming the reconstruction objective. Each architecture was trained $10$ times with different random initializations to assess consistency and reproducibility.

Unlike the point cloud experiments where ground truth principal directions provided an absolute reference, MNIST evaluation focuses on three complementary criteria: (1) consistency of latent modes across independent training runs, quantified by computing pairwise cross-correlations between modes extracted from different model instantiations; (2) disentanglement, assessed through targeted latent traversals where individual dimensions are ignored while others remain fixed; and (3) interpretability, evaluated by whether learned modes correspond to recognizable visual features that domain experts can meaningfully describe. We visualize latent modes both as images and through targeted reconstructions that isolate the contribution of individual dimensions to the overall data representation.

\subsubsection*{Standard Convolutional Autoencoder Results} The baseline convolutional autoencoder architecture, despite achieving low reconstruction error, exhibits severe deficiencies in latent space organization. Across the independent training runs, latent modes demonstrate high variability with no consistent correspondence between dimension indices and interpretable features. Computing the cross-correlation $|\langle R_i^{(a)}, R_j^{(b)}\rangle|$ for modes extracted from different runs $(a)$ and $(b)$, we find that while dimensions occasionally align by chance, there is no systematic pattern ensuring that, for example, dimension $3$ in one training run captures similar features to dimension $3$ in another run. Visualizations of the learned modes as images show that they uniformly exhibit distributed, non-localized patterns that fail to isolate coherent visual features.

Ablation experiments further confirm the absence of hierarchical structure. When reconstructing test images using only the first $k$ latent dimensions (by zeroing out dimensions $k+1$ through 7), reconstruction quality does not improve monotonically or predictably as $k$ increases. Instead, we observe irregular performance curves where omitting certain seemingly arbitrary dimensions causes larger degradation than omitting others, and this pattern varies across training runs (Fig ~\ref{fig:MNIST_error}). This behavior indicates that the network has not learned to prioritize information allocation; all dimensions contribute roughly equally to reconstruction, precluding interpretable importance ranking.

\begin{figure}[h!]
    \centering
    \includegraphics[width=0.45\linewidth]{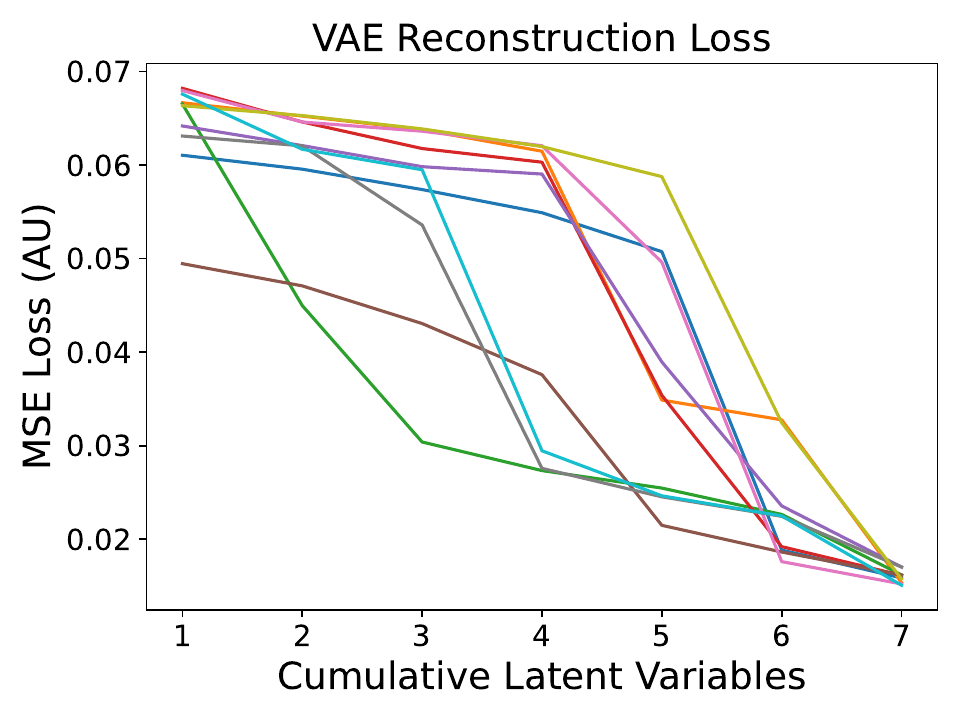}
    \includegraphics[width=0.45\linewidth]{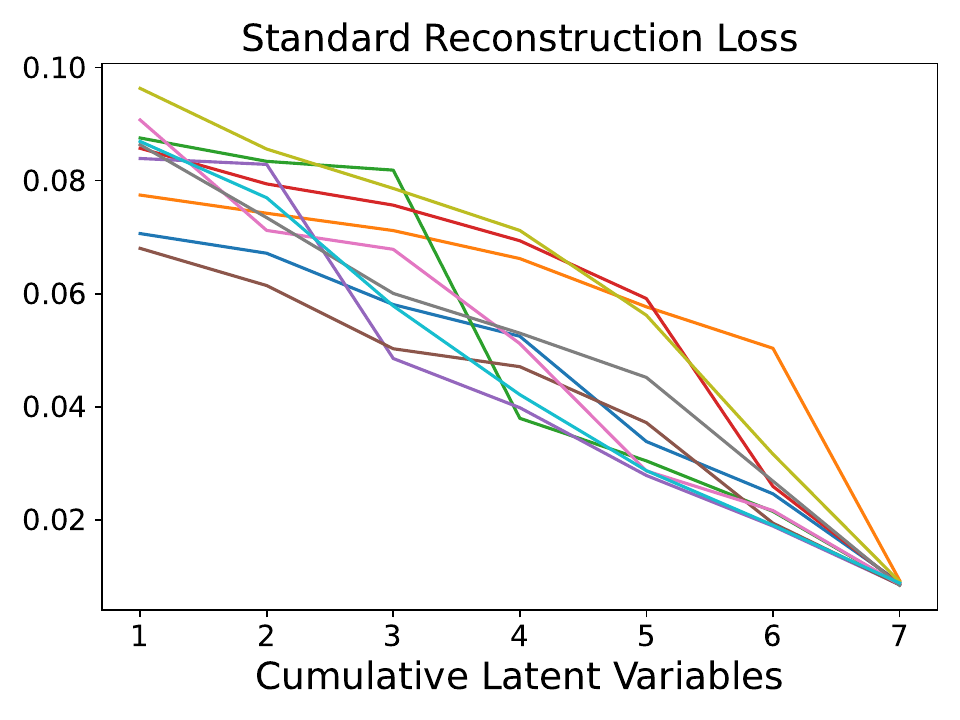}
    \includegraphics[width=0.45\linewidth]{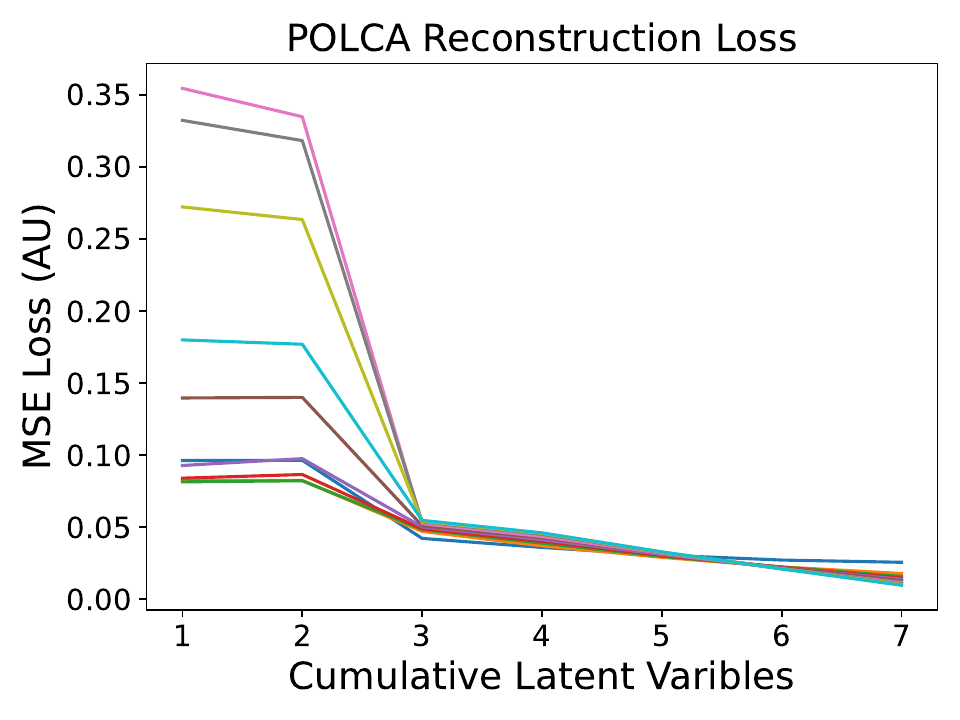}
    \includegraphics[width=0.45\linewidth]{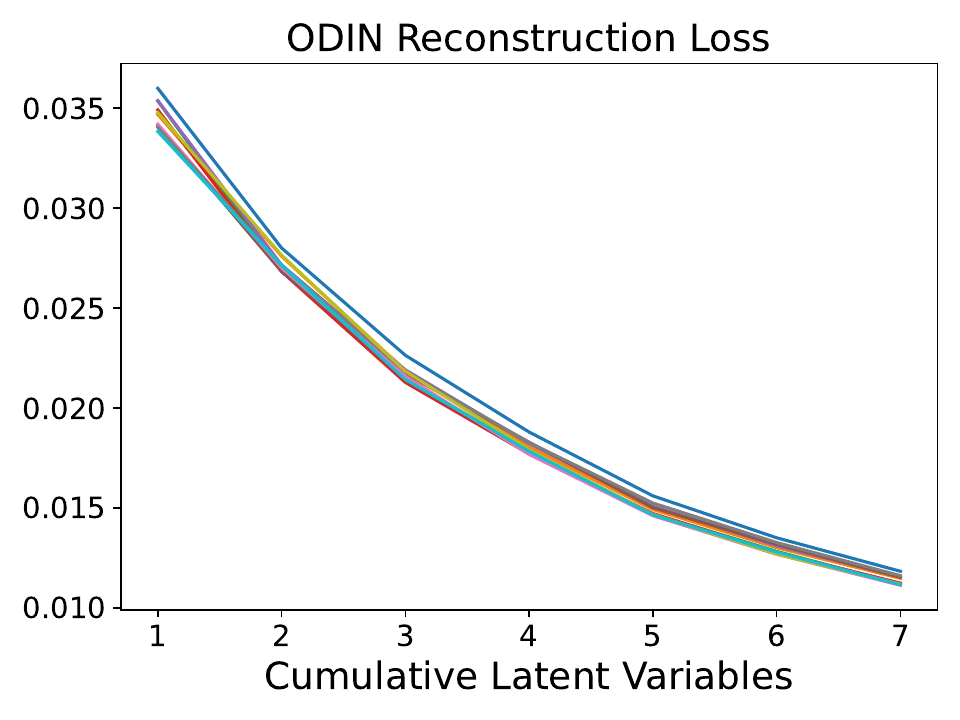}
    \caption{Reconstruction loss as a function of the number of latent dimensions (dendrites) across 10 independent runs for the VAE, Standard-AE, POLCA and ODIN architectures. Each curve corresponds to an independent training run. While Standard and VAE networks exhibit non-monotonic, run-dependent loss trajectories, ODIN produces nearly identical, monotonically decreasing curves across all runs, demonstrating substantially greater stability and predictability in latent space utilization. POLCA is able to concentrate more information into earlier latent variables due to its center of mass loss term, however, the reconstruction and classification accuracy is slightly worse when using all latent features compared to the other three models.}
    \label{fig:MNIST_error}
\end{figure}

\subsubsection*{$\beta$-Variational Autoencoder Results} The $\beta$-VAE architecture demonstrates marked improvement in latent space structure compared to the standard autoencoder, successfully achieving feature disentanglement through its probabilistic regularization framework. With $\beta=1$, the model learns latent dimensions that respond selectively to distinct generative factors. Systematic traversals along individual dimensions produce interpretable variations in reconstruction. For example, varying dimension 2 while holding others constant might systematically alter the slant of digit strokes, while dimension 3 controls stroke thickness. This disentanglement represents a significant advance for interpretability, as researchers can now identify which latent variables encode which semantic attributes.

However, while dimensions are successfully disentangled, they lack any inherent ordering or hierarchy. The learned latent space exhibits permutation symmetry—dimensions can be arbitrarily reordered without affecting reconstruction quality. Empirically, this manifests as frequent permutations across training runs. For example, dimension $3$ in one instantiation might capture digit identity (a high-importance feature), while in another run, this same information migrates to dimension $5$. Cross-correlation analysis of modes extracted from independent training runs reveals that although each run produces well-disentangled representations, the assignment of features to specific dimension indices varies unpredictably. This inconsistency severely limits the model's utility for comparing results across experiments or performing systematic ablation studies, as one cannot reliably identify the ``third most important feature'' when this designation changes with each training run.

Ablation experiments with $\beta$-VAE confirm the absence of importance ordering. When plotting reconstruction accuracy using progressive latent dimensions, the improvement curve is neither smooth nor consistent across training runs. See figure~\ref{fig:MNIST_error}. More critically, the subset of dimensions that proves most informative varies between model instantiations, confirming that the architecture provides no mechanism for consistently prioritizing high-variance directions into early latent dimensions.

\subsubsection*{ODIN Architecture Results} The ODIN architecture demonstrates exceptional performance across all evaluation criteria, successfully combining the disentanglement advantages of $\beta$-VAE with robust importance ordering and cross-run consistency that neither baseline achieves. Across all 10 independent training runs, ODIN produces latent modes that are both remarkably stable in their assignments and interpretably organized in hierarchical order of importance. The cross-correlation analysis shows that for any two independently trained ODIN models $(a)$ and $(b)$, the mode alignment satisfies $|\langle R_i^{(a)}, R_i^{(b)}\rangle|>0.999$ for all $i \in \{1,...,10\}$, indicating that dimension $i$ captures essentially identical features across all training runs. This consistency enables reproducible science as researchers can meaningfully compare ``dimension $3$'' across different experiments with confidence that this dimension represents the same underlying data variation pattern.

The importance ordering emerges naturally from ODIN's dendritic decoder architecture. Reconstruction error curves for the dendritic outputs $Dec(Z_{[1]}), \dots ,Dec(Z_{[7]})$, show smooth, monotonic improvement as additional latent dimensions are incorporated, with the steepest error reduction occurring when the first few dimensions are added. See figure~\ref{fig:MNIST_error}. Specifically, the first latent dimension alone achieves mean squared error of $0.036$, incorporating the second reduces this to $0.027$, the third to $0.021$, and by the fifth dimension, the error has dropped to $0.015$, nearly matching the full model's performance. This curve validates that ODIN successfully and reliably allocates high-variance, reconstruction-critical information to early dimensions, with later dimensions contributing progressively finer details.

\subsubsection*{Explainability Through Hierarchical Feature Analysis}

The evaluation framework extends the latent mode analysis methodology from Section 4.1 to the image domain. For each trained model, we compute latent modes $R\in \mathbb{R}^{k\times n}$ via the matrix operation $ R = Z^T X$, where $Z\in \mathbb{R}^{N\times k}$ represents the encoded latent representations for $N$ training samples and $X\in \mathbb{R}^{N\times n}$ is the matrix of flattened image vectors with $n = 784$ pixels $(28\times 28)$. Each row of $R$ corresponds to a latent dimension and can be reshaped into a $28\times 28$ image representing that dimension's characteristic pattern or ``mode'' of variation.

In the case of ODIN restricted to digits $1$ and $2$ only, the hierarchical feature structure becomes particularly transparent and reveals crisp, interpretable patterns (see Figure~\ref{fig:ODIN_modes}). The first mode captures the dominant source of variation in the dataset—the fundamental differences between digit classes. When visualized as a $28\times 28$ image, this mode shows clear digit-like structure, representing the average contrast between high-variance digit pairs. Subsequent modes isolate progressively more subtle features such as orientation angles, stroke widths, loop curvatures, and positional offsets.

\begin{figure}[h!]
    \centering
    \includegraphics[width=\linewidth]{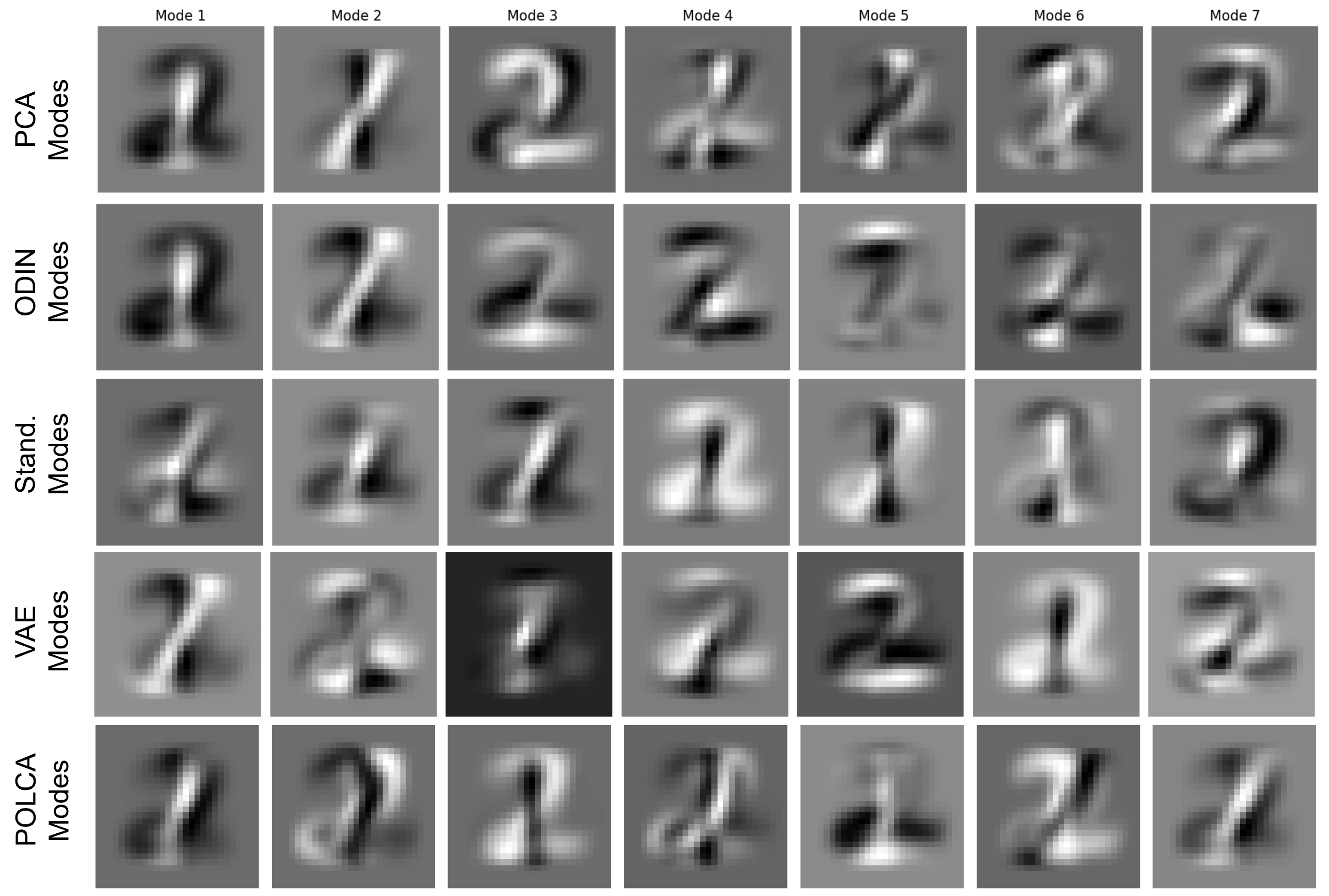}
    \caption{Panels display the first seven latent modes $R = Z^\top X$ (reshaped as $28\times 28$ images) on representative training runs for each method. ODIN modes demonstrate remarkable stability, closely aligning with PCA principal components in early dimensions while revealing hierarchical non-linear structure in later dimensions, with consistent assignment across all runs. Standard autoencoder modes spread dominant digit class separations diffusely across latent features, exhibiting high inconsistency across runs with smeared patterns that fail to isolate coherent digit characteristics. $\beta$-VAE modes achieve disentanglement with recognizable independent features, but suffer from inconsistent feature assignments across initializations. 
    }
    \label{fig:ODIN_modes}
\end{figure}

Visualizing the first latent mode by computing the average images for samples with $z_1>0$ versus $z_1<0$ reveals clear separation, with positive values corresponding to digit $2$ (with its characteristic curved top and bottom), and negative values corresponding to digit $1$ (vertical stroke structure). To quantify this, we evaluate the discriminative capacity of $z_1$ using the area under the receiver operating characteristic curve (AUC-ROC), treating $z_1$ directly as a ranking score without any learned classifier. This metric measures the probability that a randomly selected digit 2 sample receives a lower score than a randomly selected digit 1 sample, and yields an average AUC of $0.9942$, confirming near-perfect linear separability along this single dimension (see table~\ref{table:AUC_score}). Reconstructions using only $Dec(Z_{[1]})$ produce crude but recognizable digit shapes that capture the basic structural templates—vertical for $1$, curved for $2$—while omitting stylistic details.

\begin{table}[h!]
\centering
\begin{tabular}{|c|c|c|c|c|} 
 \hline
 ODIN & PCA & Standard-AE & VAE & POLCA\\ [0.5ex] 
 \hline
 0.9942 & 0.9794 &  0.9638 & 0.7012 & 0.8595\\ 
 [1ex] 
 \hline
\end{tabular}
\caption{AUC-ROC scores achieved by the first latent dimension $z_1$ as a linear separator for digits $1$ and $2$ on held out test data, averaged across multiple runs of each unsupervised network. An AUC near $1.0$ indicates that the first latent dimension captures nearly all class-discriminative structure.}
\label{table:AUC_score}
\end{table}

Subsequent latent dimensions, having captured the primary inter-class variation in dimension 1, are free to encode the dominant intra-class variation. For digit $1$, dimension $2$ manifests as angular orientation or slant. By selecting samples of digit $1$ and sorting them by $z_2$ values, we observe a smooth progression from left-leaning strokes $(z_2<0)$ through vertical orientations $(z_2 \approx 0)$ to right-leaning slants $(z_2>0)$. For digit $2$, dimension $3$ distinguishes between ``curly'' versus ``flat'' variants, \emph{i.e.} samples with $z_3<0$ exhibit tightly curved bottom loops that extend below the baseline, while $z_3>0$ corresponds to flatter, more angular bottom segments that terminate at the baseline. Higher-order dimensions continue this hierarchical decomposition into progressively finer features.

\subsection{Nitrogen Vacancy Diamond Photoluminescence Data}
To evaluate ODIN's performance on real-world scientific data, we applied both standard autoencoder and ODIN architectures to nitrogen-vacancy (NV) center spectroscopy measurements from diamond samples. The experimental setup and physical interpretation of NV photoluminescence (PL) spectra is discussed in detail by Rajpal et al \cite{rajpal_PL} in a forthcoming manuscript. The authors carefully evaluated the suitability of ML models for mapping observed spectra onto temperature. They note that while all ML models outperform expert knowledge derived linear models, autoencoders performed worst amongst the ML models. This underperformance was reasoned to arise from the inability of autoencoder models to discriminate against stochastic ``dark variables'' that are not correlated with temperature. Here we test, Rajpal et al's \cite{rajpal_PL}dark variable hypothesis using ODIN. 

\subsubsection*{Data pre-processing and  Model}
Briefly, the dataset consisted of photoluminescence spectra collected over a time series during which the sample temperature was continuously increased from 243 K to 343 K. Each spectrum was acquired at 2 nm resolution over the 613-800 nm wavelength range, resulting in 1340-dimensional input vectors. The dataset comprised 144,000 consecutive measurements, with temperature recorded simultaneously via an independent thermister. The dataset was split into $80\%$ training and $20\%$ testing for all models.

Both autoencoder architectures employed identical network configurations with a $1340\to1350\to650\to350\to175\to56\to20\to5$ encoder structure and symmetric decoder, with ReLU (Rectified Linear Unit) activation functions. The 5-dimensional latent space was chosen to provide sufficient capacity for capturing spectral variations while maintaining meaningful dimensionality reduction. Networks were trained using the Adam optimizer with stepped learning rate with plateau stopping criteria. For ODIN, the orthogonality regularization weight $\lambda_{\text{orth}}$ in~\ref{eq:total_loss} was varied from $0-3.0$ to compare latent correlations. To assess consistency, we performed 10 independent training runs for each architecture with different random weight initializations.

To evaluate the degree to which each network encodes temperature information in its latent representation, we perform a post-training linear regression analysis using the independently measured temperatures as ground truth labels. Critically, these labels are never provided to either network during training; both architectures are trained purely to minimize spectral reconstruction error, with no supervised signal of any kind. The encoder is therefore used solely to project each spectrum onto its corresponding latent representation, and only after training is complete do we introduce the temperature labels for the purpose of evaluation. Linear regression models with different subsets of the latent space, from individual latent variables up to all dimensions simultaneously, are fit to assess how much temperature information each variable carries independently, as well as how many dimensions are collectively required to achieve accurate temperature prediction.

\subsubsection*{Results}
For standard autoencoders, the distribution of temperature-related information varied substantially between runs. The analysis revealed that temperature information was often scattered across latent dimensions, with the specific dimensions varying unpredictably across initializations (see Figure~\ref{fig:standard correlations}). In some runs, dimensions 1, 3, and 5 jointly explained $90\%$ of temperature variance; in others, dimensions 2 and 4 were required to achieve comparable performance.

ODIN, in contrast, consistently encoded temperature information in the second latent dimension across all 10 training runs, see Figure~\ref{fig:odin correlations}. ODIN consistently achieved $R^2>0.80$ using only the first two latent variable and $R^2>0.90$ using the first three latent variables in 9 of 10 runs. Standard autoencoders required an average of 3.2 latent variables to reach $R^2>0.80$ and 4.5 variables for $R^2>0.90$.

Despite the stark difference in latent space organization, both architectures achieve
comparable temperature predictions on testing data (see Figure~\ref{fig:regression errors}). This might initially seem to undercut the case for ODIN, but it is precisely what one would expect: if both networks successfully capture temperature-relevant information, a sufficiently flexible regression should recover it regardless of how that information is distributed.\footnote{We remark that incorporating pairwise interaction (cross) terms between latent variables reduces the prediction error to within approximately half a degree, with ODIN achieving marginally better performance than the standard autoencoder. While the inclusion of cross terms somewhat complicates direct interpretation, it is consistent with the physical expectation that mean spectral intensity and temperature-dependent shifts interact non-linearly.} The distinction lies not in raw information content but in the accessibility and reliability with which that information can be extracted and interpreted.

\begin{figure}[h!]
    \centering
    \includegraphics[width = 0.48\linewidth]{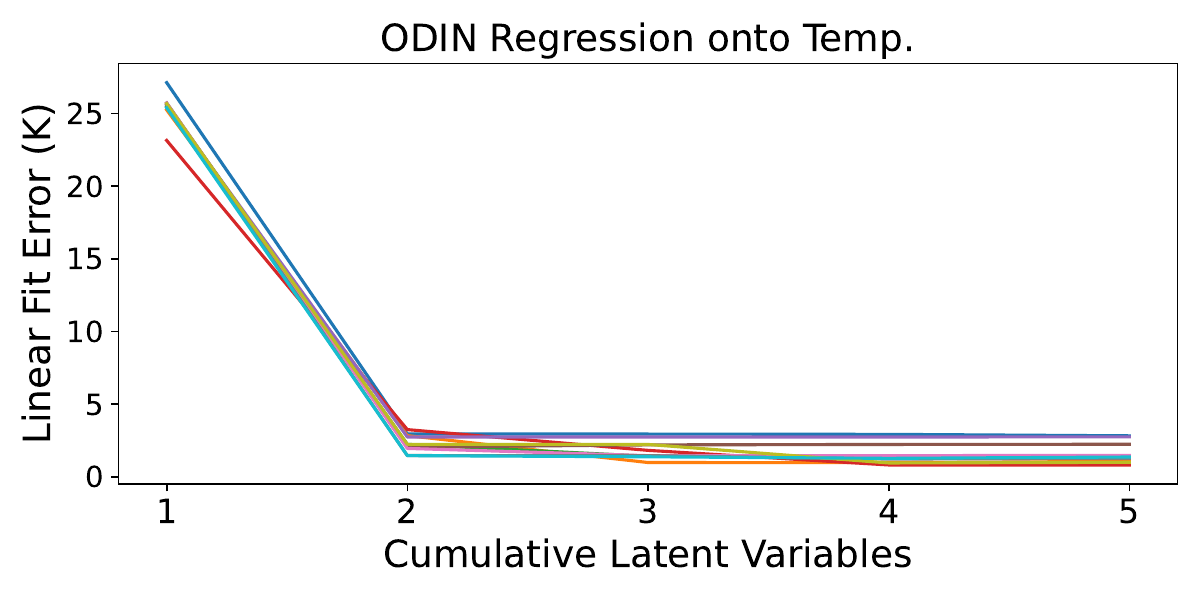}
    \includegraphics[width = 0.48\linewidth]{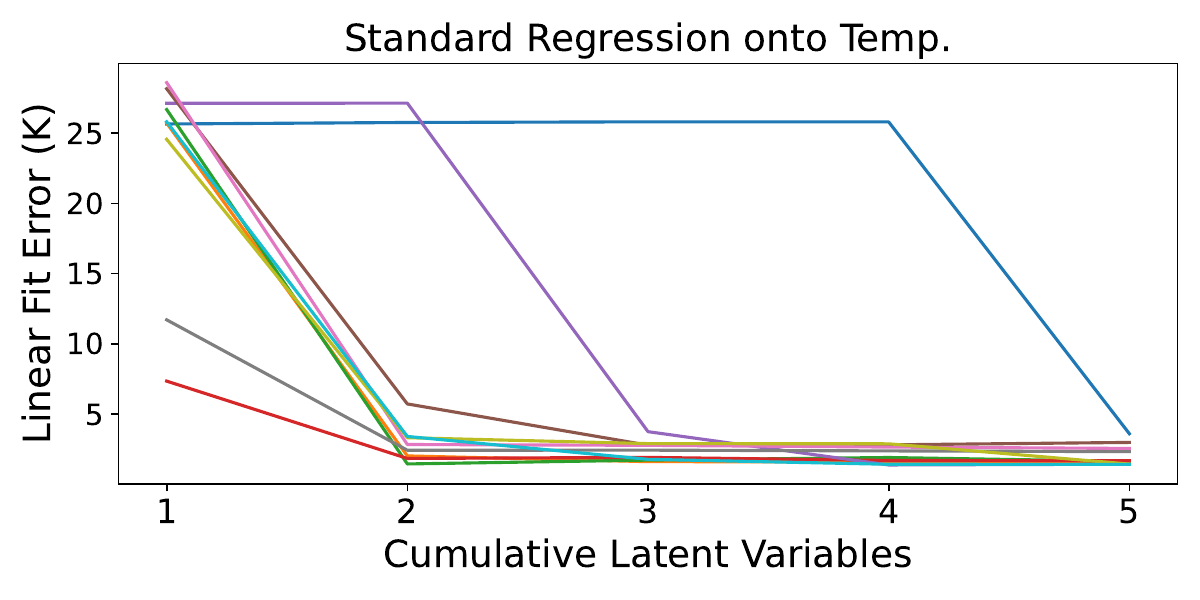}
    \caption{Linear regression error (in Kelvin) on held-out test data as a function of the number of latent dimensions included in the regression, for ODIN (left) and standard autoencoder (right) networks. Each curve corresponds to an independent training run. ODIN consistently achieves a sharp reduction in error at the second latent dimension, reflecting its stable assignment of temperature-dependent spectral variation to dimension two across all runs. The standard autoencoder exhibits substantial run-to-run variability, with temperature information distributed inconsistently across dimensions depending on initialization.}
    \label{fig:regression errors}
\end{figure}

\subsubsection*{Physical Interpretation of Latent Organization}

A closer examination of the learned latent representations reveals a striking contrast in how the two architectures resolve spectral features. Across all independent training runs, ODIN consistently assigns its first latent dimension to encode changes that are strongly correlated with the mean of the spectrum observations over time. We refer to this mode as the ``mean latent mode''. The mean latent mode represents the single largest source of variance in the dataset \emph{in terms of reconstruction}. We hypothesize that this mode is capturing overall intensity changes in the overall PL profile. The second latent dimension is strongly correlated with temperature changes (see Figure~\ref{fig:odin correlations}).  Temperature-dependent shift in the zero-phonon line (ZPL) and broadening of the ZPL and first two phonon modes of the broader phonon side band captured by the second dimension, while physically significant, constitute a comparatively smaller perturbation. ODIN's hierarchy faithfully respects this structure.

\begin{figure}
    \centering
    \includegraphics[width = 0.8\linewidth]{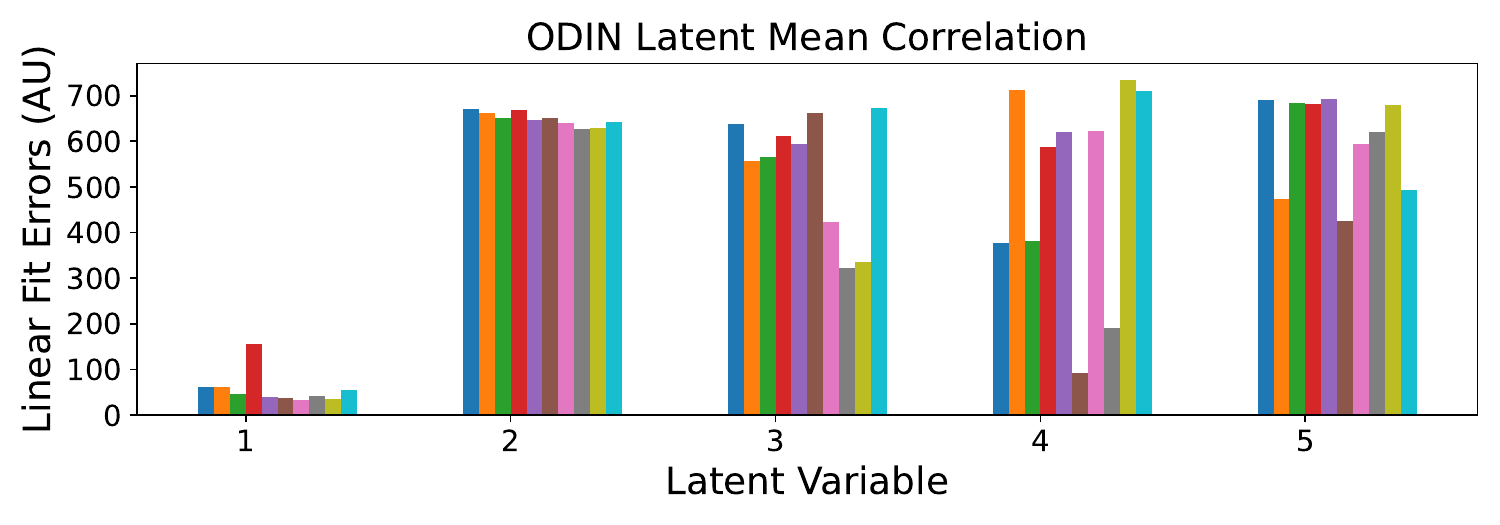}
    \includegraphics[width = 0.8\linewidth]{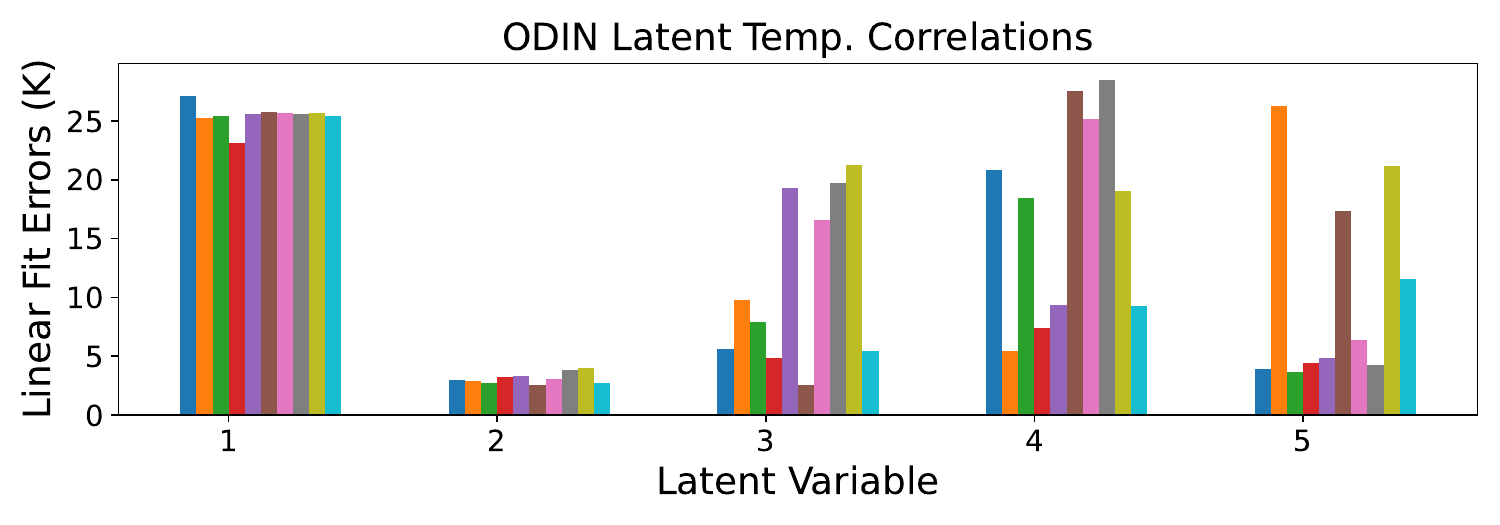}
    \caption{Bar plots showing the linear fit correlation of individual ODIN latent dimensions with the dataset mean spectrum (top) and measured temperature (bottom), for each of several independent training runs (distinguished by color). Dimension one is consistently and strongly correlated with the mean spectrum across all runs, while dimension two consistently carries the dominant correlation with temperature. Higher-order dimensions show weak and variable correlations with both quantities.}
    \label{fig:odin correlations}
\end{figure}

Visualization of the latent space confirmed this interpretation: projecting the data onto ODIN latent dimensions revealed a clear temperature gradient along dimension two, while dimension one captured measurement-to-measurement intensity variations, reflecting the variation in total collected photoluminescence (see Figure~\ref{fig:odin_vars}), although it does over short temperature intervals appear to be correlated with temperature.

\begin{figure}
    \centering
    \includegraphics[width=0.8\linewidth]{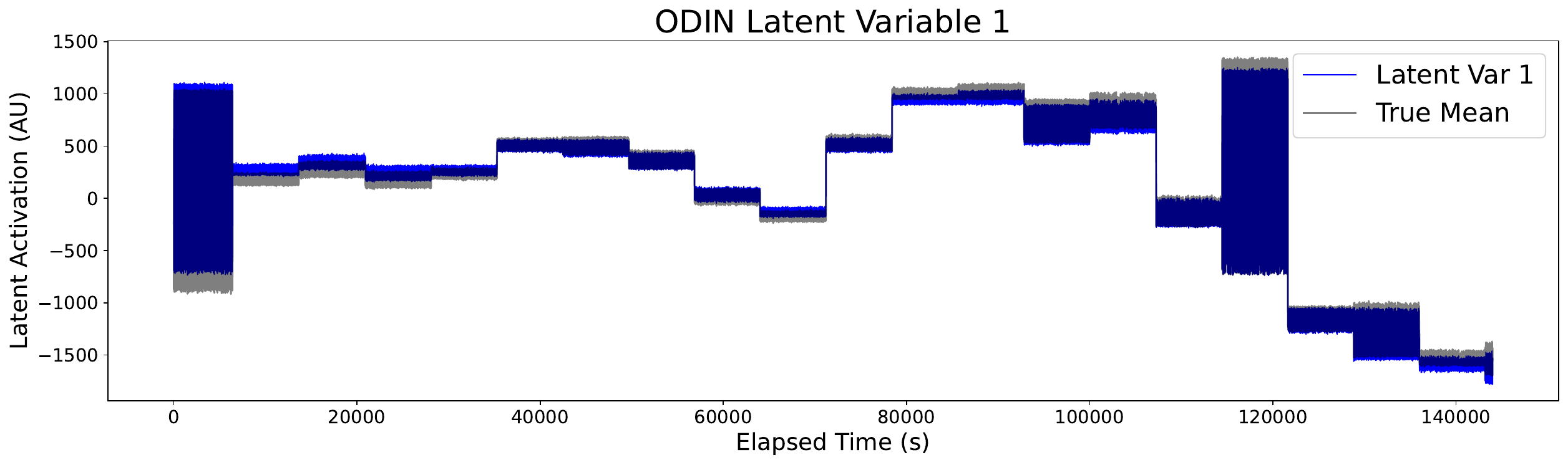}
    \includegraphics[width=0.8\linewidth]{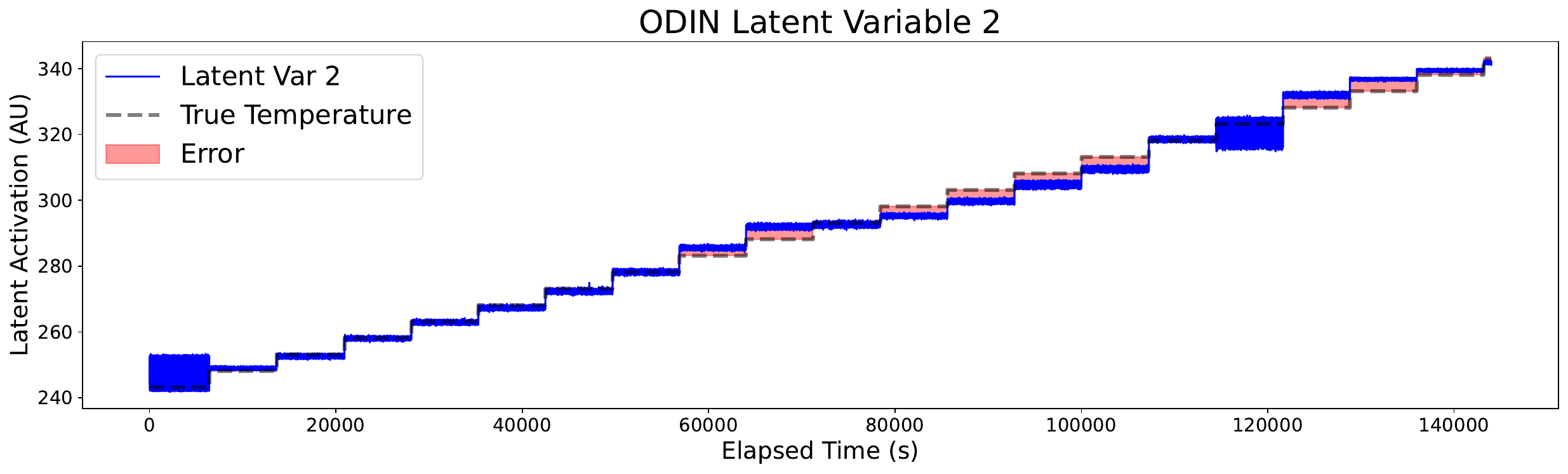}
    \caption{The first two latent dimensions learned by ODIN, shown for a representative training run, alongside the dataset true mean spectrum (top) and temperature (bottom) for comparison. Latent dimension one (top) is strongly correlated with the true mean spectrum, capturing the baseline photoluminescence profile common to all measurements. Latent dimension two (bottom) is strongly correlated with temperature, reflecting the thermally driven shift and broadening of the zero-phonon line and phonon sideband. Both latent modes are displayed after affine rescaling according to the coefficients obtained by regressing each mode onto its respective physical quantity. This organization was consistent across all independent training runs.}
    \label{fig:odin_vars}
\end{figure}

Standard autoencoders, by contrast, do not exhibit this clean separation. Across training runs, most (and in some cases, all) latent dimensions show substantial correlation with the mean spectrum, with additional dimensions appearing to encode corrective factors on top of this dominant signal (see Figure~\ref{fig:standard correlations}). Rather than isolating independent sources of variation, the standard network spreads mean latent mode-related information diffusely across the latent space, with temperature information entangled within this distributed representation. This behavior is consistent with the broader paradigm of unconstrained autoencoders: without an explicit mechanism to overcome the symmetry of the solution space, the network finds no reason to avoid redundant encodings, and the resulting latent space reflects arbitrary superpositions of the true underlying factors rather than a clean decomposition.

\begin{figure}
    \centering
    \includegraphics[width = 0.8\linewidth]{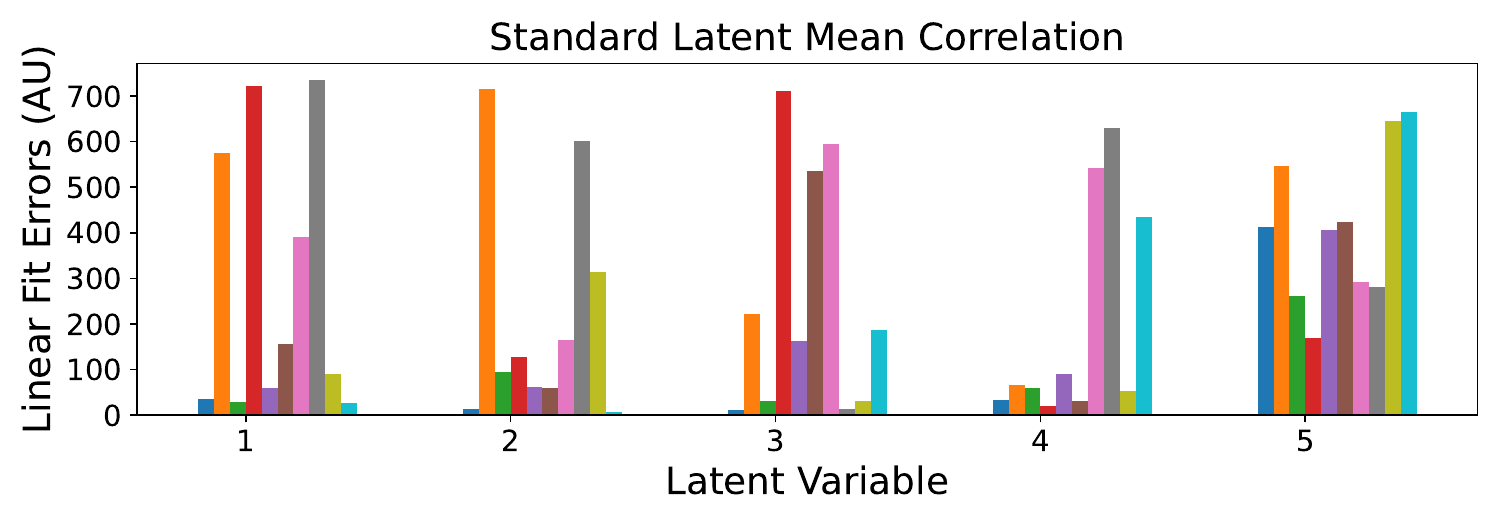}
    \includegraphics[width = 0.8\linewidth]{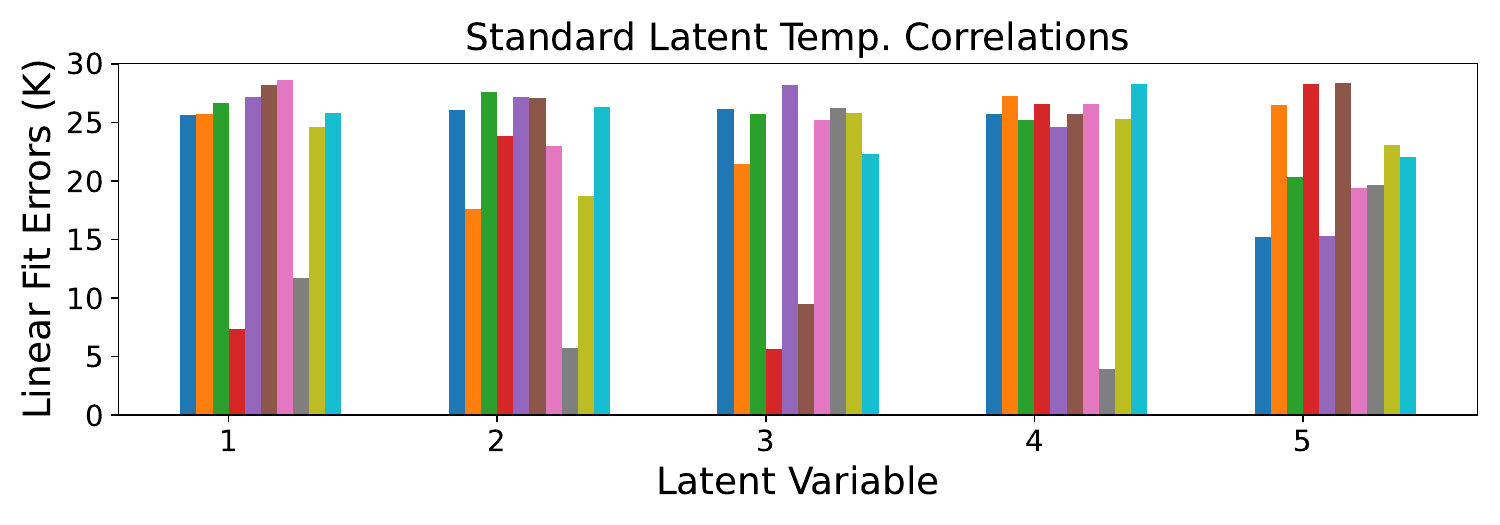}
    \caption{Linear fit correlation of standard-AE latent dimensions with the dataset mean spectrum (top) and with temperature (bottom), shown for multiple independent training runs (colored bars). In contrast to ODIN, mean-spectrum information is spread across multiple dimensions and varies substantially from run to run, reflecting the absence of any ordering in the learned latent space. Temperature correlation similarly fails to localize to a consistent dimension, appearing in different positions depending on the random initialization.}
    \label{fig:standard correlations}
\end{figure}

The relationship between the mean spectrum and temperature prediction deserves elaboration. When the mean of each observation is subtracted from the dataset prior to training, ODIN correctly promotes temperature to its first latent dimension. However, this mean-centered configuration yields slightly worse regression performance onto temperature (while maintaining comparable reconstruction quality), suggesting that the mean spectrum carries some  temperature-predictive information. This is physically plausible: in fiber-coupled NV sensing configurations, the measured photoluminescence spectrum is a composite signal mainly dominated by NV-center fluorescence, but also containing a nontrivial background contribution from the optical fiber itself and the epoxy and other materials used in fiber sensor assembly. Both the NV fluorescence and fiber contributions are proportional to the incident laser power. Hence random fluctuations or long-term drift in laser power can lead to spectrum-to-spectrum variations observed for mean latent mode projections that have temporal correlations but no temperature correlation.  The correlation between the latent mean mode and  temperature likely arises due to temperature dependent changes in the fiber's (and associated material's) background florescence intensity and fiber assembly's transmission properties. Other likely, though weaker, contributions could arise from temperature dependent changes in NV$^{-}$ to NV$^{0}$ charge state conversion

In a recent study, \cite{johansson2026characterization} have carefully examined the contribution of various types of optical fiber's background florescence (sometime referenced as autoflorescence) in NV sensing application. The authors demonstrate that optical fibers show a broad fluorescence band above $600$ nm centered near $644–645$ nm region that is associated with non-bridging oxygen hole center (NBOHC) defects in fused-silica glass. It significantly overlaps with the NV florescence spectra. Crucially, the NBOHC fluorescence intensity is itself temperature-dependent due to phonon-electron coupling of the NBOHC center to the silica matrix. This suggests that the overall spectral weight and baseline intensity (what ODIN captures as its first latent mode) is not a temperature-independent nuisance variable, but rather a partially informative proxy for thermal state of the sensor assembly.    
Higher-order ODIN dimensions, those beyond the first two, capture structured residual variation that is neither mean-like nor purely temperature-dependent. The fact that these dimensions do not contribute substantially to improved regression suggests they are not as critical for temperature prediction, but may nonetheless encode subtle experimental covariates or higher-order effects that are physically meaningful to understand the evolution of PL spectra over time. These effects may include slight changes in the fiber bundle's light coupling and transmission properties due to either temperature history, laser power exposure strain due to bending, changes in optical properties of the epoxy used in sensor and or thermo-mechanical changes in the epoxy/fiber sensor assembly.

When such degradation accumulates over time, as such it constitutes an archetypal dark variable; structured, physically motivated, but uncontrolled. Because the dominant signals have been cleanly extracted into the leading modes, these residual dimensions are far easier to interrogate than the analogous dimensions of a standard autoencoder, where mean and temperature information has not been fully separated from whatever else the network has learned. In this sense, ODIN's ordered decomposition transforms latent space exploration from an arbitrary search into a principled, progressive investigation of the data's underlying structure. For the NV spectroscopy application, this provides a stable coordinate system for analyzing how spectral features evolve with thermal excitation. Given the evidence presented in~\cite{johansson2026characterization}, a natural first hypothesis to test is whether slow drift in the fiber's autofluorescence background manifests as structured variation in the higher-order ODIN dimensions. This could be directly probed by acquiring long-duration fiber emission spectra under controlled, temperature-stable conditions and examining whether the resulting drift signature projects preferentially onto the residual latent modes.

\subsubsection*{Cross Encoder Validation}

\begin{figure}
    \centering
    \includegraphics[width=0.8\linewidth]{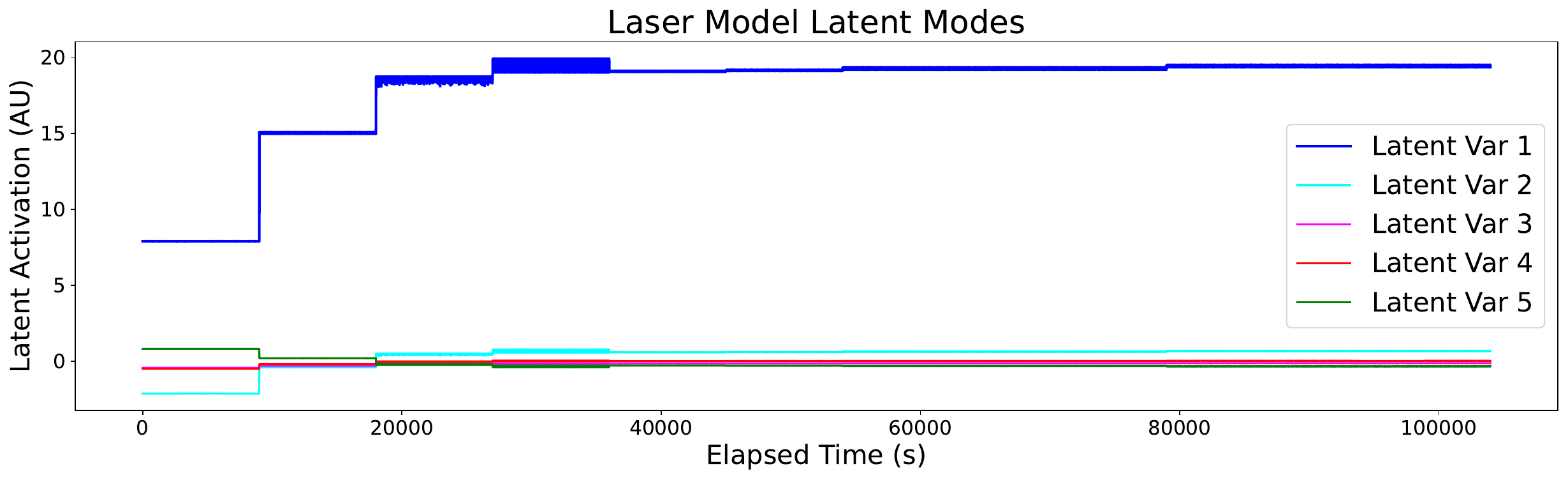}
    \includegraphics[width=0.8\linewidth]{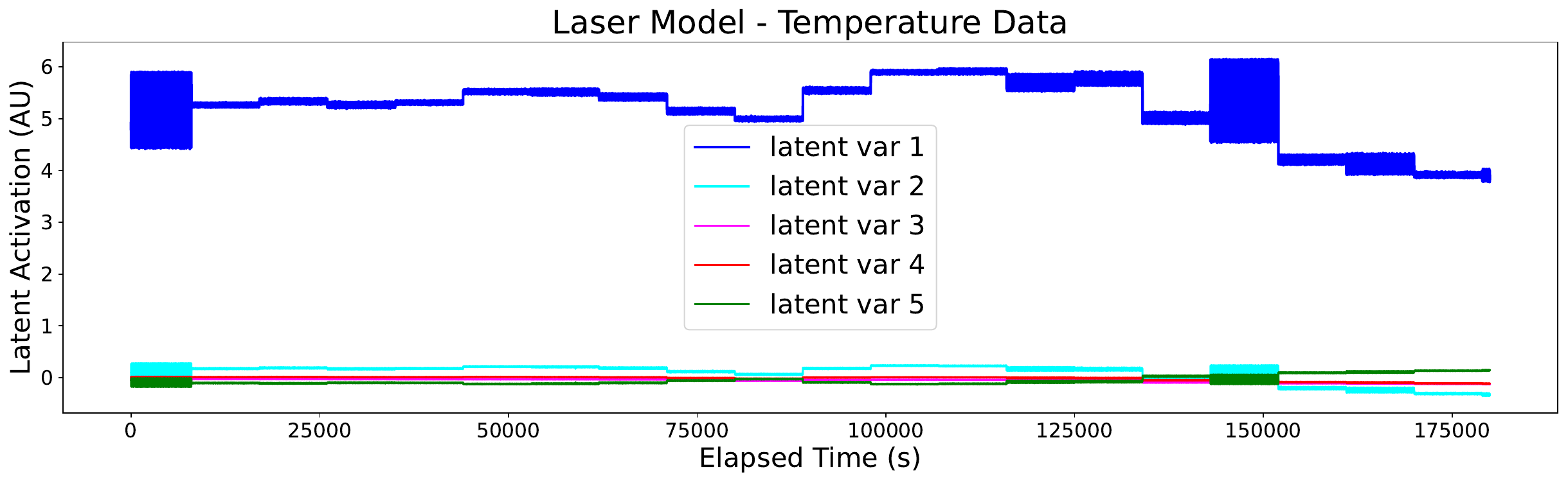}

    \caption{(Top) Latent modes of the laser-trained encoder layer of ODIN for in-distribution laser power data: one dominant mean-like mode, with higher-order modes flat and near-zero. (Bottom) Cross-passed temperature data through the laser encoder of ODIN: mean mode active, higher-order modes remain inert. The dominance of the mean mode across all conditions confirms that laser power variation occupies a single intensity-scaling dimension in spectral space.}
    \label{fig:laser_modes}
\end{figure}

A central claim in the preceding analysis is that ODIN's first latent mode encodes overall spectral weight rather than a temperature-specific signal. Under this hypothesis, any physical perturbation that modulates total photon flux should drive the first latent mode, regardless of whether it also modulates spectral shape. To test this directly, we collected laser power variation data under controlled, fixed-temperature conditions and examined the response of both the temperature-trained encoder and the laser power-trained encoder when presented with either dataset.

Training ODIN on laser power variation data at fixed temperature produces a particularly clean latent structure. One dominant mode emerges, corresponding to the row mean of the data, which is expected given that laser power and total emission intensity are approximately proportional and the spectral shape is largely invariant to laser power under fixed-temperature conditions. The higher-order modes learned by the laser encoder are notably flat, carrying negligible residual structure regardless of whether laser data or temperature data are passed through them (see Figure~\ref{fig:laser_modes}). This flatness indicates that laser power variation, at fixed temperature, exhausts its spectral degrees of freedom in a single dominant direction, with no meaningful higher-order content remaining once the mean scaling is accounted for. The same flatness persists when temperature data are passed through the laser encoder, confirming that the laser encoder has no coordinate axes onto which the richer spectral structure of temperature variation can project. In both directions, the shared signal between the two datasets is exactly and only the row mean. 

Passing laser power data through the temperature encoder predominantly activates the mean-related mode in a manner consistent with the known proportionality between incident laser power and total emission intensity, while the temperature-sensitive mode and higher-order dimensions exhibit only weak, mean-like residual activations. This is a strong indicator that the first latent mode of the temperature encoder is not a specialist temperature coordinate, but a general intensity coordinate. Higher order modes carry information that is genuinely absent from fixed temperature, variable laser-power observations.

The activation of the mean mode across both encoder-dataset combinations identifies it as the unique latent coordinate shared by all physical perturbations. The mean-like residual activations produced when laser power data are passed through the temperature encoder suggests that the temperature encoder's higher-order modes are not completely orthogonal to intensity variation in the cross-passed setting, but that the residual response they do produce carries no structured spectral content beyond a broad scaling. This is distinct from the structured phonon-sideband and ZPL character seen in those modes during temperature data encoding, and it is consistent with the interpretation that those modes are specialist coordinates for spectral shape variation that simply find no matching structure in the laser dataset, leaving only a weak, unstructured mean-like residual.

This interpretation has practical consequences for preprocessing choices. Mean-subtracting each spectrum prior to training removes the thermal information carried by the mean mode along with the laser-power-driven intensity variance. When the goal is temperature inference, this tradeoff is unfavorable, and the experimental results reported here are consistent with the observation that training on un-mean-subtracted spectra yields marginally improved temperature regression performance.

The present results, obtained entirely without labels, raise a natural question about what could be achieved with access to even partial supervision. The current ODIN encoders are trained on single experimental conditions and learn modes that reflect the dominant sources of variance within each condition. A supervised or semi-supervised extension of ODIN, in which known physical labels such as laser power and temperature are incorporated as auxiliary targets on selected latent dimensions, could in principle further factorize the latent distribution. Supervision onto known physical variables anchors those dimensions explicitly, leaving the residual modes free to absorb any remaining structured variation. Subsequent targeted data collection under controlled conditions would then provide the labeled data needed to either confirm or rule out candidate dark-variable hypotheses, and potentially to introduce those variables as additional supervisory targets. Exploring this direction, and in particular whether ODIN's orthogonality structure generalizes reliably across experimental conditions when extended to joint training on multiple specialist datasets, remains an important avenue for future work.

\section{Conclusions}

In this work, we introduced ODIN, a novel autoencoder architecture that addresses the rotational ambiguity and feature entanglement inherent to standard autoencoder optimization. By enforcing orthogonality and ordered variance decomposition through dendritic decoding, ODIN produces latent representations that are statistically independent, sorted by explained variance, and reproducible across experiments. Unlike $\beta$-VAEs, which require careful, data-dependent tuning of a balancing hyperparameter, ODIN achieves disentangled representations and faithful reconstructions simultaneously, without delicate trade-offs.

A central theoretical contribution of this work is the result that, in the linear regime, ODIN provably recovers PCA via the reconstruction pathway, establishing a rigorous foundation that connects the architecture to a well-understood and optimal solution. Rather than indirectly encouraging ordering structure through auxiliary losses or hyperparameter schedules, ODIN inherits its organizational principles directly from its architecture, making the nonlinear generalization both natural and theoretically motivated. Each latent component maintains a stable semantic correspondence to a direction of data variation, empowering researchers to identify what the model has learned, in what order of priority, and with what degree of statistical independence.

This transparency supports an iterative workflow that bridges machine learning and the traditional scientific method. Researchers can train ODIN on initial data, interrogate the latent structure to surface dominant factors and unexpected patterns, and design targeted follow-up experiments guided by those findings. For the broader goal of integrating machine learning into scientific workflows, ODIN exemplifies the importance of model interpretability as a prerequisite for scientific utility. By transforming the autoencoder from a purely predictive tool into an instrument for hypothesis testing and exploratory analysis, ODIN provides a principled and flexible foundation for dimensionality reduction in complex, high-dimensional domains. We hope it serves as a useful framework for researchers who require not just accurate representations, but ones they can understand, trust, and build upon.


\acks{Tyrus Berry and Jeanie Schreiber acknowledge support from NIST Award 70NANB24H232 and Zeeshan Ahmed acknowledges support from AFMETCAL Award R24-685-0005.}


\vskip 0.2in
\bibliography{bibliography}

\vskip 0.2in
\appendix
\section{}
\label{app:PCA}


This appendix collects proofs of the key results about PCA stated in the main text. We begin in section~\ref{subsec:no_order} by characterizing the solution set of the standard (unordered) PCA objective, showing that any projection onto the subspace spanned by the top $k$ right singular vectors is a maximizer. We then show in section~\ref{subsec:with_order} that this rotational ambiguity is resolved by the ordered objective whose minimizer recovers the principal components in their natural order of importance.

\subsection{Review of Standard PCA Results}
\label{subsec:no_order}

The PCA projection is typically characterized by the following maximum variance projection problem,
\begin{equation}\label{maxvar} \max_W  \text{Tr}(W^T X^T X W), \quad \text{subject to  } W^T W = I_{k\times k}\end{equation}
where $\text{Tr}(W^T X^T X W) = ||XW||_F^2$ is the variance of the projection into the latent space.  
The orthogonality constraint is $W^TW - I = 0$, so we have the following Lagrangian formulation:
  \[\mathcal{L}(W,\Lambda) = \text{Tr}(W^TX^TXW) - \text{Tr}(\Lambda(W^TW-I)).\]  
To characterize solutions of the above maximization, first note the following:
\begin{lemma}
    Only the symmetric part of $\Lambda$ affects the Lagrangian $\mathcal{L}(W,\Lambda)$.  The Lagrangian is also invariant to orthogonal transformations applied to the right side of $W$.
\end{lemma}

\begin{proof}
    Any square matrix $\Lambda$ can be decomposed uniquely as $\Lambda = \Lambda_s+\Lambda_a$, where $\Lambda_s = \frac{1}{2}(\Lambda+\Lambda^T)$ and $\Lambda_a = \frac{1}{2}(\Lambda-\Lambda^T)$, with $\Lambda_s^T = \Lambda_s$ (symmetric) and $\Lambda_a^T = -\Lambda_a$ (skew-symmetric). Since $S = W^TW-I$ is symmetric and $\Lambda_a$ is skew symmetric, we have
    \begin{align*}
        \text{Tr}(\Lambda_a(W^TW-I)) &= \text{Tr}((\Lambda_aS)^T)\\&= \text{Tr}(S^T\Lambda_a^T)\\&=
        \text{Tr}(S(-\Lambda_a))\\&=
        -\text{Tr}(\Lambda_aS),
    \end{align*}
so the trace must be zero. Thus the Lagrangian only depends on $\Lambda_s$.

Next, consider an orthogonal transformation $Q$. For any $W$ satisfying $W^TW = I$, we have that $\tilde W = WQ$, also satisfies $\tilde W^T\tilde W = I$. Moreover, the PCA objective \eqref{maxvar} satisfies
\[\text{Tr}(W^TX^TXW) = \text{Tr}(\tilde W^TX^TX\tilde W).\]
Thus PCA as characterized by \eqref{maxvar} is right-orthogonally invariant and solutions are only unique up to rotations and reflections.
\end{proof}

Now we can derive the following Lagrangian formulation of the PCA subspace defined up to orthogonal transformation:
\begin{theorem}Let $X \in \mathbb{R}^{N \times n}$, with singular value decomposition $X=USV^\top$, then 
\[ \sum_{i=1}^k s_i^2 = \max_{W \in \mathbb{R}^{n \times k}} {\rm Tr}(W^\top X^\top X W) + {\rm Tr}(\Lambda(W^\top W - I)) \]
where $s_i$ are the singular values of $X$ in decreasing order, and the maximum is achieved when $W=V_{[k]}Q=VI_{n\times k}Q$ where $Q \in \mathbb{R}^{k\times k}$ is any orthogonal matrix.
\end{theorem}
\begin{proof}
Taking the derivative of the Lagrangian we have:
\[\frac{\partial\mathcal{L}}{\partial W} = 2X^TXW - 2W\Lambda= 0\implies X^TXW = W\Lambda.\]
 By the previous lemma we may assume $\Lambda$ is symmetric, in particular it can be diagonalized as $\Lambda = QDQ^T$ ($D$ diagonal), and we may freely choose the basis where $\Lambda$ is diagonal. In that case, the stationarity condition becomes:
\begin{align*}
    X^TX\tilde W &= X^TX(WQ)
    \\&= (X^TXW)Q\\&=
    (W\Lambda)Q\\&=
    W(QDQ^T)Q\\&=
    (WQ)D\\&=
    \tilde WD
\end{align*}
which is an eigenproblem where each column of $\tilde W$ is an eigenvector of $X^TX$, \emph{i.e.} the right singular vectors of $X$. In order to maximize the trace objective, the columns of $\tilde W\in\mathbb{R}^{n\times k}$ should contain the leading $k$ singular vectors of $X$, which is the classical result from PCA. In that case we have \[\tilde W= WQ = V_{[k]}\implies W = V_{[k]}Q^\top = VI_{n\times k}Q^\top\]
and \[\max_{W^\top W = I_{k\times k}} ||X W||_F^2 = \|XV_{[k]}Q^\top\|_F^2 = \|USI_{n\times k}Q^\top\|_F^2 = \sum_{i=1}^ks_i^2\]
\end{proof}
Note that the above Lagrangian gives a projection $W$ into the PCA subspace spanned by the leading $k$ right singular vectors. However, this optimization problem only recovers the PCA subspace (since $W$ is only defined up to an orthogonal transformation $Q$), and does not order the the intrinsic coordinates in the latent space.

\subsection{PCA With Ordering}
\label{subsec:with_order}
To recover the principal components in their natural order, we augment the reconstruction objective with a sequence of nested terms indexed by the columns of $W$. The resulting objective breaks the rotational symmetry of the unordered problem and, as we show below, admits a minimizer equal to the ordered principal components $v_1, v_2, \dots, v_k$ up to sign change. Consider first the following important lemma:

\begin{lemma}
\label{lem:orth_projector}
    If $W \in \mathbb{R}^{n\times k}$ is a global minimizer of the function
\[
    F(W)=\|X^\top-WW^\top X^\top\|_F^2
\] then $WW^\top$ is the orthogonal projector onto $\mathrm{range}(W)$. If additionally $W$ is full rank, then it's columns are orthonormal.
\label{lem:projection}
\end{lemma}
\begin{proof}
Note that $\|X^\top-WW^\top X^\top\|_F^2
=\|(I-WW^\top)X^\top\|_F^2$. Using the identity \(\|M\|_F^2=\mathrm{tr}(MM^\top)\) with $M:=(I-WW^\top)X^\top,$
then
\begin{align*}
\|(I-WW^\top)X^\top\|_F^2
&=\mathrm{tr}( M M^\top) \\
&=\mathrm{tr}((I-WW^\top)X^\top\,((I-WW^\top)X^\top)^\top).
\end{align*}
Compute the transpose:
\[
((I-WW^\top)X^\top)^{\top}
= X (I-WW^\top)^\top
= X (I-WW^\top),
\]
since \(I-WW^{\top}\) is symmetric. Letting $S:=X^\top X\in \mathbb{R}^{n\times n}$, we have
\begin{align*}
\|(I-WW^{\top})X^\top\|_F^2
&=\mathrm{tr}((I-WW^\top)X^\top X(I-WW^\top)) \\
&=\mathrm{tr}((I-WW^\top)S(I-WW^\top)).
\end{align*}
Then use cyclicity of the trace to move the left factor to the right:
\[
\mathrm{tr}((I-WW^{\top})S(I-WW^{\top}))
=
\mathrm{tr}((I-WW^\top)(I-WW^\top)S)
=
\mathrm{tr}((I-WW^\top)^2S).
\]

Thus
\begin{align*}
\|X^{\top}-WW^{\top}X^{\top}\|_F^2
&=\|(I-WW^{\top})X^{\top}\|_F^2 \\
&=\mathrm{tr}((I-WW^{\top})^2\,S).
\end{align*}


For any solution $W$ consider $\mathcal U:=\operatorname{range}(W)$. Let $P = P_\mathcal{U}$ be the orthogonal projector onto $\mathcal U$, \emph{i.e.} $P^2 = P$ and $Px = x$ for all $x\in \mathcal{U}$. Since $A:=WW^\top$ maps into $\mathcal U$, we have $(I-P)A=0$, and since $A$ is symmetric we also have $A(I-P)=0$. Hence
\[
I-A=(I-P)+(P-A)
\quad\text{with}\quad
(I-P)(P-A)=0=(P-A)(I-P).
\]
So the cross terms in $(I-A)^2$ vanish and we have 
\[
(I-A)^2=(I-P)^2+(P-A)^2=(I-P)+(P-A)^2,
\]
and so \[\mathrm{tr}((I-A)^2S)=
\mathrm{tr}((I-P)S)
+
\mathrm{tr}((P-A)^2S).\]

It is not hard to show that $\mathrm{tr}(MN)\geq 0$ for two symmetric PSD matrices $M,N$ (since $\mathrm{tr}(MN) = \mathrm{tr}(M^{1/2}NM^{1/2})$ and $x^\top (M^{1/2}NM^{1/2})x = (M^{1/2}x)^\top N(M^{1/2}x)\geq 0$). Then since $S = X^TX$ is symmetric PSD and $(P-A)^2$ is symmetric PSD (since $x^\top(P-A)^2x = ((P-A)x)^\top ((P-A)x) = \|(P-A)x\|^2\geq 0$), we have that $\mathrm{tr}((P-A)^2S)\geq 0$ and hence \begin{align*}
\mathrm{tr}((I-A)^2S)
&=
\mathrm{tr}((I-P)S)
+
\mathrm{tr}((P-A)^2S) \\
&\geq
\mathrm{tr}((I-P)S).
\end{align*}
Equality holds iff $\mathrm{tr}((P-A)^2S)=0$. In particular, if $S$ is positive definite on $\mathcal U$ this forces $A=P$. In other words, for a fixed subspace $\mathcal{U}$, the first term is minimized when $WW^\top$ is exactly the orthogonal projector $P$. 

Now consider the columns of $W$. Each column $w_j\in \mathcal{U}$ is fixed by $WW^\top = P_\mathcal{U}$, \emph{i.e.} $WW^\top w_j = w_j$ for each $j$. Collecting these identities column-wise gives 
    \[WW^\top W = W\implies W(W^\top W-I_k) = 0.\]
    If $W$ is full rank, it's null space is trivial and we obtain $W^\top W-I_k = 0$, and thus $W^\top W = I_k$. 

\end{proof}

With lemma~\ref{lem:orth_projector} in hand, the equivalence between PCA maximum trace and minimum reconstruction objectives is straightforward:
\begin{theorem}
    For any $X\in \mathbb{R}^{N\times n}, W\in \mathbb{R}^{n\times k}$ we have the following equivalence:\[\max_{W^\top W = I}  \text{Tr}(W^\top X^\top X W) = \min_W \| X^\top - W W^\top X^\top \|_F^2\]
\end{theorem}

\begin{proof}
    Consider the Frobenius norm of $X^\top - WW^\top X^\top$. By lemma~\ref{lem:orth_projector}, this is minimized when $WW^\top$ is the orthogonal projector onto the range of $W$, \emph{i.e.} $(WW^\top)^2 = WW^\top$. Hence
    \begin{align*}
        \|X^\top - WW^\top X^\top\|_F^2 &= \text{Tr}\left[\left(X^\top-WW^\top X^\top\right)^\top\left(X^\top-WW^\top X^\top\right)\right]\\&=
        \text{Tr}\left[\left(X-XWW^\top\right)\left(X^\top-WW^\top X^\top\right)\right]
        \\&=
        \text{Tr}\left[XX^\top - XWW^\top X^\top - XWW^\top X^\top + X(WW^\top)^{\cancel{2}} X^\top\right]
        \\&= \text{Tr}(X^TX) - 2\text{Tr}(W^\top X^\top XW) + \text{Tr}(W^\top X^\top XW)\\&=
        \|X\|_F^2 - \text{Tr}(W^\top X^\top XW)
    \end{align*}
    Using the fact that the trace of a product of matrices is invariant under cyclic permutations. Since $\|X\|_F^2$ is constant with respect to $W$, we thus have \[\min_{W} \| X^\top - W W^\top X^\top \|_F^2 \implies \max_{W^\top W = I}  \text{Tr}(W^\top X^\top X W).\] Starting from $\max_{W^\top W = I}  \text{Tr}(W^\top X^\top X W)$, a similar argument follows in reverse using $W^\top W = I$. Thus we have 
    \[\min_{W} \| X^\top - W W^\top X^\top \|_F^2 \iff \max_{W^\top W = I}  \text{Tr}(W^\top X^\top X W).\]
\end{proof}

Now we are ready to prove the main theorem for ordered PCA:
    
\begin{theorem}
\label{thm:dend_loss}
    Let $X \in \mathbb{R}^{N \times n}$, with singular value decomposition $X=USV^\top$ and distinct singular values $s_1>s_2>\cdots>s_k>0$, then 
\[ k s_1^2 + (k-1)s_2^2 + \cdots + 2 s_{k-1}^2 + s_k^2 = \min_{W \in \mathbb{R}^{n \times k}}\sum_{i=1}^k  \|X^\top-W_{[i]}W_{[i]}^\top X^\top\|_F^2, \]
and the minimum is achieved when $W=V_{[k]}\Sigma=VI_{n\times k}\Sigma$, where $\Sigma = \mathrm{diag}(\varepsilon_1, \dots, \varepsilon_k)$, with $\varepsilon_j = \pm 1$.
\end{theorem}

\begin{proof}
    Let $J(W):= \sum_{i=1}^k  \|X^\top-W_{[i]}W_{[i]}^\top X^\top\|_F^2$ be the error of the full summed prefix objective and $m_i:=\mathrm{min}_{W_{[i]}}\|X^\top-W_{[i]}W_{[i]}^\top X^\top\|_F^2$ the minimum of individual $i$-dimensional PCA objectives. For each fixed $i$, we know that $m_i$ is attainted when $W_{[i]} = V_{[i]}Q_i$, however, this does not guarantee the optimality of any earlier prefixes.

    Instead of optimizing over each $i$-dimensional subspace independently, the minimizer of $J(W)$ is a single shared matrix $W$ whose prefixes must simultaneously define optimal subspaces of dimensions $1, \ldots,
    k$. This means that instead of choosing the rotations $Q_1, \ldots, Q_k$ independently, they must be selected simultaneously so that optimality is preserved at every prefix level.

    In general $J(W)\geq \sum_{i=1}^k m_i$, and this coupling may initially appear to severely restrict the set of viable solutions. However, we already know that such a solution exists, namely $W = V_{[k]}$. In that case, $W_{[i]} = V_{[i]} \quad \text{for all } i = 1, \ldots, k,$ and since $\mathrm{span}(V_{[i]})$ is the unique optimal $i$-dimensional principal subspace under the distinct singular value assumption, each prefix $W_{[i]}$ attains $m_i$ individually. Therefore \[J(V_{[k]}) = \sum_{i=1}^k m_i,\] and combined with the universal lower bound $J(W) \geq \sum_{i=1}^k m_i$, this establishes that $V_{[k]}$ is a global minimizer of $J$.
    
    Thus, any global minimizer $W^*$ of the full $J$ must satisfy that every prefix $W^*_{[i]}$ solves the $i$-dimensional PCA problem. Since $s_1>\dots >s_k>0$, the top $k$ right singular subspaces are distinct and we obtain $\mathrm{span}(W^*_{[i]}) = \mathrm{span}(V_{[i]})$ for all $i=1,\dots, k$. Moreover, each $W^*_{[i]}$ is a global minimizer of the $i$-th term, so by lemma~\ref{lem:orth_projector} $W^*_{[i]}(W^*_{[i]})^\top$ is an orthogonal projector onto $\mathrm{range}(W^*_{[i]})$; in particular, $W^*_{[i]}$ has orthonormal columns.
    
    To show exactness of solutions $V_{[k]}\Sigma$, we argue inductively on $i$. For $i = 1$, $\mathrm{span}(w_1^*) = \mathrm{span}(v_1)$, hence $w_1^* = \pm v_1$. Assume $w_j^*= \pm v_j$ for $j<i$. Since $\mathrm{span}(W^*_{[i]}) = \mathrm{span}(V_{[i]})$ and $\mathrm{span}(W^*_{[i-1]}) = \mathrm{span}(V_{[i-1]})$, orthonormality implies \[w_i^*\in \mathrm{span(V_{[i]})\cap\mathrm{span(V_{[i-1]})^\perp}} = \mathrm{span}(v_i),\]
    so $w_i^* = \pm v_i$. Thus $W^* = V_{[k]}\Sigma$, where $\Sigma = \mathrm{diag}(\varepsilon_1, \dots, \varepsilon_k)$, with $\varepsilon_j = \pm 1$.
\end{proof}


\section{}
\label{app:ODIN}
This appendix establishes the theoretical behavior of the ODIN loss functions in the linear regime. In general we have two different ways to proceed: either we can augment standard reconstruction minimization with an orthogonality term to recover an unordered version of $V$, or we can break the symmetry with dendrites to recover $V$ with the correct ordering. 

Section~\ref{subsec:no_dend} treats the baseline case of plain reconstruction combined with an orthogonality penalty, showing that this recovers the principal components but without any guaranteed ordering of components. Section~\ref{subsec:with_dend} then shows that replacing the reconstruction term with a dendritic reconstruction (even without any orthogonality constraint) is sufficient to recover the principal components in order of importance. Finally, section~\ref{subsec:combined} considers the full ODIN objective combining both dendritic reconstruction and orthogonality regularization, confirming that ordered recovery is preserved under the complete loss. See table~\ref{tab:solutions} for an overview of closed-form optimal solution sets in the linear regime.

\begin{table}[]
    \centering
    \begin{tabular}{|c|c|c|c|}
    \hline
        PCA/MSE & MSE + Orth & Ordered PCA/Dendrites & Dendrites + Orth\\
        \hline
         $V_{[k]}Q$& $V_{[k]}P\Sigma $& 
         $V_{[k]}\Sigma$ &$V_{[k]}\Sigma$\\ 
    \hline
    \end{tabular}
    \caption{Optimal solutions for considered loss objectives in the linear regime. Solutions progress from under-constrained to fully identified: standard PCA recovers any orthogonal rotation of $V_{[k]}$, adding an orthogonality penalty narrows this to any column permutation of $V_{[k]}$, the dendrite loss alone (and combined with orthogonality) uniquely identifies $V_{[k]}$ up to sign.}
    \label{tab:solutions}
\end{table}

\subsection{Plain MSE with Orthogonality}
\label{subsec:no_dend}
We begin with the simpler linear setup consisting of a single reconstruction term paired with an orthogonality penalty on latent correlations. This combination is the most natural starting point, as the orthogonality constraint alone might be expected to pin down the solution. However, as we show below, the resulting objective recovers only the principal components and does not enforce any ordering among solutions.

In this setting, the loss function takes the form:
\begin{equation}
\label{eq:no_dend}
    \mathcal{L}_{\text{MSE}} + \mathcal{L}_{\text{orth}}= \min_{W}\|X^\top-WW^\top X^\top\|_F^2+\sum_{i<j}(W^\top X^\top XW)_{ij}^2
\end{equation} and is minimized by some permutation of $V$.

First we show that this combined loss obtained by adding the orthogonality term is also minimized by an `orthogonal projector'.

\begin{lemma}
    If $W \in \mathbb{R}^{n\times k}$ is a global minimizer of the function
\[
    F(W)=\|X^\top-WW^\top X^\top\|_F^2+\sum_{i<j}(W^\top X^\top XW)_{ij}^2
\] then $WW^\top$ is the orthogonal projector onto $\mathrm{range}(W)$. 
\label{lem:projection_orth}
\end{lemma}
\begin{proof}

As in lemma~\ref{lem:projection}, choose $ W\in\mathbb{R}^{n\times k}$ with orthonormal columns spanning $\mathcal{U}=\mathrm{range}(W)$ to minimize the first term. Now, for the second term, set $B:= W^{\top}S W$, which is symmetric. By the spectral theorem, there exists orthogonal $Q$ such that
\[
Q^\top BQ=\Lambda
\quad\text{is diagonal}.
\] Then we have
\[
(Q^\top W^\top) S (WQ) = W'^\top S W' = \Lambda,
\]
so all off-diagonal entries vanish and the second term becomes $0$, i.e.
\[
\sum_{i<j}(W'^\top S W')_{ij}^2=0.
\]
Moreover, the orthogonal transformation $W\mapsto W'= WQ$ does not affect the first term since, for any $Q\in \mathbb{R}^{k\times k}$ orthogonal, we have $\|X^T-WW^TX^T\|_F^2 = \|X^\top-(WQ)(WQ)^\top X^\top\|_F^2$. Thus, the ``orthogonal projector'' does not increase the first term, and the subsequent orthogonal rotation $ W\mapsto W'= WQ$ keeps the first term fixed while making the second term as small as possible (indeed $0$).
\end{proof}

The following lemma will be useful in showing that the solution to (\ref{eq:no_dend}) is any permutation of $V$:
\begin{lemma}
    Let $S = \text{{Diag}}(s_1, \dots, s_n)$ be a diagonal matrix with distinct $s_i\ne0$ and suppose that $Q$ is any orthogonal matrix satisfying $Q^\top SQ = D$, where $D$ is any full rank diagonal matrix. Then $Q$ is a permutation matrix.
    \label{lem:diagonals}
\end{lemma}

\begin{proof}
    The hypothesis $Q^\top SQ = D$ is exactly the condition
    \[q_i^\top Sq_j = \sum_{k=1}^n s_kq_{ki}q_{kj} = 0, \quad\forall i\ne j.\]
    If $Q$ is orthogonal, we have $\sum_{k}q_{ki}q_{kj} = 0$ for $i\ne j$. Multiplying by $s_1$ and subtracting from the above equation we get
    \[\sum_{k=2}^n(s_k-s_1)q_{ki}q_{kj} = 0.\]
    We can repeat this process and since all $s_i$ are distinct, all $(s_k-s_1), (s_k-s_2), \dots$ are non-zero. This forces $q_{ki}q_{kj} = 0$ for all $k$ and $i\ne j$. This means that for each row $k$ of $Q$, at most one column has a nonzero entry in that row. If we additionally assume that $D$ is full rank, then it is invertible and $Q$ must be invertible also. Since invertible matrices cannot have all columns/rows completely zero, we conclude that $Q$ has exactly one non-zero entry per-column per-row. Since $Q$ is orthogonal this non-zero entry must be equal to $\pm1$, giving the form of a permutation matrix. 
\end{proof}

We are now ready to prove the result:
\begin{theorem}
    Let $X \in \mathbb{R}^{N \times n}$, with singular value decomposition $X=USV^\top$ and distinct singular values $s_1>s_2>\cdots>s_k>0$, then 
    \[\sum_{i=1}^k s_i^2 = \min_{W \in \mathbb{R}^{n \times k}}  \|X^\top - WW^\top X^\top\|+ \sum_{i<j}(W^\top X^\top XW)_{ij}^2 \]
    where $s_i$ are the singular values of $X$ in decreasing order, and the minimum is achieved when $W=V_{[k]}A$ for some permutation matrix $A$.
\end{theorem}

\begin{proof}
    Note that the second term is minimized when $W^\top X^\top XW$ is diagonal, \emph{i.e.} \[W^\top X^\top XW = W^\top VS^2V^\top W = D,\] for some diagonal matrix $D$. By lemma~\ref{lem:diagonals}, $A = V^\top W$ is a permutation matrix precisely when $W$ is orthogonal. However, by lemma~\ref{lem:projection_orth}, we may assume that any global minimizer $W$ may take the form that yields the orthogonal projector $W W^\top = P_{\mathcal{U}}$, with $ \mathcal{U} = \mathrm{range}(W)$. In particular, if $W$ is full rank then $W W^\top = P_{\mathcal{U}}$ necessarily implies that it's columns are an orthonormal basis for $\mathcal{U}$. Thus $W = V_{[k]}A$, where $A$ is a permutation matrix. This is the PCA projection onto the first $k$ principal components, though not in any particular order. 
\end{proof}

\subsection{Dendrites with no Orthogonality}
\label{subsec:with_dend}

Having seen that orthogonality alone is insufficient to order the principal components, we now show that the dendritic reconstruction objective achieves this ordering on its own, even in the absence of any orthogonality constraint. Recall: 
    \[\mathcal{L}_{\text{Dend}}= \sum_{i=1}^k\mathcal{L}_{\text{MSE}}(X,\hat{X}_i)= \min_W \sum_{i=1}^k \|X^\top-W_{[i]}W_{[i]}^\top X^\top\|_F^2.\]
    
This is the same loss function as in theorem~\ref{thm:dend_loss}, which was shown to recover the true PCA solution. Unlike the unordered objective considered above, the nested structure of the dendritic terms breaks the rotational symmetry of the solution space, and the minimizer in the linear regime coincides with $V_{[k]}$.

\subsubsection{Lower Bound for Dendritic MSE}

One practical consequence of a neural network implementation of dendritic structure is the behavior of the loss during training. The nested structure of the reconstruction terms imposes a sequence of lower bounds on the achievable error at each stage, bounding the total aggregate MSE away from zero.

Let $X \in \mathbb{R}^{N\times n}$ with covariance matrix $\Sigma_X$ and eigenvalue decomposition
\[
\Sigma_X = P \Lambda P^\top,
\]
where $\Lambda = \operatorname{diag}(\lambda_1, \lambda_2, \ldots, \lambda_n)$ and $\lambda_1 \geq \lambda_2 \geq \cdots \geq \lambda_n \geq 0$ denote the variances along the principal components of $X$. Since the principal components contribute additively to total variance, any $k$-dimensional reconstruction $\hat X_k$ satisfies
\[
\mathrm{MSE}(X, \hat{X}_k) \geq \sum_{i=k+1}^{n} \lambda_i.
\]
Unless the data lie entirely within the span of some $k$-dimensional subspace, this bound is strictly positive. Otherwise, any method that reconstructs $X$ using a reduced latent representation of dimension $k<n$ is inevitably unable to capture the residual variance contained in the remaining components.

Applying this reasoning to the ODIN architecture, the partial reconstruction errors $L_1, L_2, \ldots, L_n$ satisfy analogous lower bounds. Specifically, $L_i\geq \sum_{j=i+1}^n \lambda_j$, and the total dendritic loss is therefore bounded below by:
\[
\mathcal{L}_{\text{Dend}} = \sum_{i=1}^n L_i \geq \sum_{i=1}^n\sum_{j=i+1}^n \lambda_j,
\]
reflecting the diminishing contribution of residual variance as successive latent dimensions are incorporated.

\subsection{Combined Loss Objective}
\label{subsec:combined}

We now analyze the full ODIN objective, combining dendritic reconstruction with orthogonality regularization. The results of the previous two sections suggest these terms play similar roles, and we confirm that their combination does not interfere with ordered recovery. We show that the minimizer of the full loss in the linear regime is again $W = V_{[k]}$ with potential sign changes.
\begin{theorem}
    Let $X \in \mathbb{R}^{N \times n}$, with singular value decomposition $X=USV^\top$ and distinct singular values $s_1>s_2>\cdots>s_k>0$, then 
    \[ k s_1^2 + (k-1)s_2^2 + \cdots + 2 s_{k-1}^2 + s_k^2 = \min_{W \in \mathbb{R}^{n \times k}}  \mathcal{L}_{\text{Dend}}+ \mathcal{L}_{\text{orth}}, \]
    and the minimum is achieved when $W=V_{[k]}\Sigma=VI_{n\times k}\Sigma$, where $\Sigma = \mathrm{diag}(\varepsilon_1, \dots, \varepsilon_k)$, with $\varepsilon_j = \pm 1$.
\end{theorem}

\begin{proof}
    Note that the second term is minimized when $W^\top X^\top XW$ is diagonal, \emph{i.e.} \[W^\top X^\top XW = W^\top VS^2V^\top W = D,\] for some diagonal matrix $D$. By lemma~\ref{lem:diagonals}, $A = V^\top W$ is a permutation matrix precisely when $W$ is orthogonal. However, by lemma~\ref{lem:projection}, we may assume that any global minimizer $W$ may take the form that yields the orthogonal projector $W W^\top = P_{\mathcal{U}}$, with $ \mathcal{U} = \mathrm{range}(W)$. In particular, since $W$ is full rank then $W W^\top = P_{\mathcal{U}}$ necessarily implies that it's columns are orthonormal.

    So we may assume that $A$ is a permutation matrix and additionally that, for fixed $i$, $W_{[i]}W_{[i]}^\top$ is the orthogonal projector onto $\text{span}\{V a_1, \dots Va_i\}$. Thus \[\|X^\top-W_{[i]}W_{[i]}^\top X^\top\|_F^2 = \|X^\top\|_F^2 - \|W_{[i]}W_{[i]}^\top X^\top\|_F^2\]
    And since $W_{[i]}^\top X^\top = A_{[i]}^\top V^\top VSU^\top = A_{[i]}^\top SU^\top$, we have
    \begin{align*}
        \|W_{[i]}W_{[i]}^\top X^\top\|_F^2 &= \text{Tr}(XW_{[i]}\cancel{W_{[i]}^\top W_{[i]}}W_{[i]}^\top X^\top) \\&= \text{Tr}(XW_{[i]}W_{[i]}^\top X^\top)\\&= \text{Tr}(A_{[i]}^\top S^2 A_{[i]})
    \end{align*}
    Thus
    \begin{align*}
        \|X^\top-W_{[i]}W_{[i]}^\top X^\top\|_F^2 &= \text{Tr}(X^\top X) - \text{Tr}(A_{[i]}^\top S^2 A_{[i]})\\&=
        \sum_{j=1}^n s_j^2 - \sum_{\ell=1}^i a_{\ell}^\top S^2a_\ell
    \end{align*}
    Summing over $i$ we have
    \[\sum_{i=1}^k \|X^\top-W_{[i]}W_{[i]}^\top X^\top\|_F^2 = k\sum_{j=1}^ns_j^2 - \sum_{\ell=1}^k(k-\ell+1)a_\ell^\top S^2a_\ell.\]
    So minimizing the reconstruction loss is equivalent to maximizing $\sum_{\ell=1}^k(k-\ell+1)a_\ell^\top S^2a_\ell$. Since $S^2$ has distinct eigenvalues, then by Lemma \ref{lem:diagonals} we have $A = P\Sigma$, where $P$ is a permutation matrix and $\Sigma$ is a sign matrix. Since the weights are decreasing $k>k-1>\dots >1$, the sum is maximized when the $a_\ell$ vectors are those which pick the largest singular values (\emph{i.e.} $A = I_{n\times k}\Sigma\implies W = VI_{n\times k}\Sigma = V_{[k]}\Sigma$)
\end{proof}

\end{document}